\documentclass{article}

\usepackage[preprint]{neurips_2023}

\usepackage[utf8]{inputenc} 
\usepackage[T1]{fontenc}    
\usepackage{hyperref}       
\usepackage{url}            
\usepackage{booktabs}       
\usepackage{amsfonts}       
\usepackage{nicefrac}       
\usepackage{microtype}      
\usepackage{xcolor}         

\usepackage{amsmath}                                  
\usepackage{mathtools}                                
\usepackage{bm}                                       
\usepackage{subcaption}                               
\usepackage{graphicx}                                 
\usepackage[nameinlink,capitalise,noabbrev]{cleveref} 
\usepackage{natbib}                                   
\setcitestyle{authoryear, open={(},close={)}}
\usepackage{amsthm}                                   
\usepackage{amssymb}                                  
\newtheorem{proposition}{Proposition}
\newtheorem{definition}{Definition}

\usepackage{gensymb}                                  
\usepackage{wrapfig}                                  
\usepackage{enumitem}                                 
\usepackage{pgfplots}
\pgfplotsset{compat=1.17}
\usepackage{xcolor}
\definecolor{light_gray}{HTML}{f0f0f0}
\definecolor{mid_gray}{HTML}{d9d9d9}
\definecolor{light_blue}{HTML}{306EFF}
\definecolor{light_orange}{HTML}{F87217}

\definecolor{neur_purpl}{HTML}{DCEBF4}
\definecolor{outt_purpl}{HTML}{ef3b2c}
\definecolor{neur_green}{HTML}{F8CFBD}

\title{Learning Linear Groups in Neural Networks}

\author{%
  Emmanouil Theodosis\\
  Harvard University\\
  \texttt{etheodosis@seas.harvard.edu}
  \And
  Karim Helwani\\
  Amazon Web Services\\
  \texttt{helwk@amazon.com}
  \AND
  Demba Ba\\
  Harvard University\\
  \texttt{demba@seas.harvard.edu}
}

\begin{document}

\maketitle

\begin{abstract}
  Employing equivariance in neural networks leads to greater parameter efficiency and improved generalization performance through the encoding of domain knowledge in the architecture; however, the majority of existing approaches require an a priori specification of the desired symmetries. We present a neural network architecture, Linear Group Networks (LGNs), for learning linear groups acting on the weight space of neural networks. Linear groups are desirable due to their inherent interpretability, as they can be represented as finite matrices. LGNs learn groups without any supervision or knowledge of the hidden symmetries in the data and the groups can be mapped to well known operations in machine learning. We use LGNs to learn groups on multiple datasets while considering different downstream tasks; we demonstrate that the linear group structure depends on \emph{both} the data distribution and the considered task.
\end{abstract}

\section{Introduction}
\label{sec:intro}
Convolutional neural networks \citep{LBD+89,LBBH98} were one of the first architectures that introduced the concept of equivariance to neural networks. By exploiting translation symmetry, networks can greatly reduce their number of learnable parameters compared to fully connected networks, while at the same time resulting in models that generalize better. While equivariance to translations is a natural choice for image classification or speech recognition, it is not the only one; planar rotations of objects leave the class identity unchanged and variations in pitch do not alter the spoken utterance. Moreover, certain data modalities, such as graphs, do not utilize translational symmetries and instead require other symmetries.

Equivariant convolutional networks \citep{CoWe16} were proposed as a method to design equivariant neural networks. However, encoding symmetries poses certain challenges: most architectures require an explicit specification of the desired symmetry, lack a general framework and require special treatment for every group, and only allow for \emph{unidirectional} knowledge flow. Indeed, baseline frameworks do not allow learning the symmetries from the data and instead they need to be specified during construction. At the same time, they are not generalizable as the network architecture needs to change fundamentally in order to account for the different groups, depending on the desired symmetry. Finally, it is possible to embed known symmetries in the model (creating a knowledge path from the designer to the network), but it is not possible to uncover symmetries in the data and extract them. This step is crucial as in many modalities we lack a clear understanding of the present symmetries.

Ideally, the equivariant structure would be learnt from the data, for the specific task. At the center front of this equivariant structure is the operation that characterizes the group: its group action. However, isomorphisms between groups can inhibit interpretability, as two groups can have the same structure but their actions transform the data in different manners. In that context, recovering a group structure that is isomorphic to a known group, for example $\mathcal{D}_4$ offers little insight without knowing how the elements operate in the data domain. We would like to recover the specific group that is interesting for the data and the task. The \emph{linear group}, the group of invertible transformations, is an extensive group that encapsulates a vast variety of groups as subgroups, including rotations, reflections, translations, and more. Under this group, group actions are clearly understood and can be visualized giving insights on how the matrices interact with the signals from the data domain.

We propose Linear Group Networks (LGNs) that learn elements of the linear group and uses them to construct cyclic groups on the weight space of neural networks. Our main contributions can be summarized as follows:
\begin{itemize}
  \item We propose a computational framework that learns elements of the \emph{general linear group} $\operatorname{GL}_d(K)$ and use them to construct sets of filters that belong to a (finite) cyclic group whose action is a linear operator. The framework is constructive and requires minimal changes to existing architectures, training procedures, and pipelines.
  \item We apply the framework on datasets of natural images and recover the group structure of the filter sets. We discover multiple structures of interest, including groups whose actions are \emph{skew-symmetric}, \emph{Toeplitz}, or act on \emph{multiple scales}.
  \item We analyze the learned actions via ablation studies and draw connections to well-known operations in machine learning. We show that certain sets of groups have actions that have a high correlation with compositions of rotations and median filtering (related to the popular pooling operations in neural networks).
\end{itemize}

Finally, we emphasize that a main contribution of our work is analyzing the learned group actions for the different filter sets when the architecture is trained on natural images, and drawing interepretable analogues to well-known operations in machine learning. Prior work focused on the ability of their networks to learn carefully curated (and well-studied) known symmetries; instead, we analyze the group structures that organically arise in ubiquitous datasets as an effort to inform our understanding of important symmetries for data modalities and different tasks, rather than provide an architecture purely for symmetry learning.

\section{Related work}
\label{sec:related}
In recent years there has been a resurgence in interest for equivariant neural networks following the concurrent works of \citet{CoWe16} and \cite{DDK16}, where convolutional frameworks that were equivariant with respect to elementary rotations and reflections were introduced. The work of \cite{KoTr18} showed that linear equivariant maps are intertwined with group convolutions and inspired a vast array of practical architectures for encoding equivariances, including architectures equivariant to arbitrary rotations \citep{WGTB17}; steerable representations \citep{CoWe17}; avoiding interpolation artifacts by modeling equivariances on the sphere \citep{EMD20}; and extensions to vector fields \citep{MVKT17}. 

A growing body of work has studied the process of learning symmetries: \citet{CoWe14,DWL+21,MSSS+22} consider learning Lie groups by utilizing the corresponding Lie algebras; \citet{BFIG20} learn distributions over data augmentations to recover invariances from the data; \citet{RoLo22} consider approximate symmetries over exact ones, as exact symmetries might be restrictive for natural data; and \citet{ZKF21} propose a meta-learning scheme for learning equivariances by reparametrizing the weight matrices. Our work is most closely related to that of \citet{ZKF21}, however our method doesn't rely on a meta-learning framework and learns the filter sets and the corresponding group actions at the same time. Most importantly, we address one of the main limitations of \citet{ZKF21}, where symmetry discovery was not possible in single-task settings. Recently, \cite{SSOH23} proposed the use of the bispectrum to learn groups and their orbits. However, their framework acts directly on the \emph{input} space, whereas ours acts on the \emph{weight} space.

\section{Preliminaries}
\label{sec:prelim}
In this section we introduce the necessary background for our method, which relies on equivariance, cyclical and linear groups, and unfolded networks. Unfolded architectures have been used in the literature for their parameter efficiency and principled derivation from optimization problems. Moreover, they have been used with great success for compressive sensing \citep{GrLe10}, state-of-the-art denoising \citep{TDB21}, downstream tasks and rank-minimization \citep{RoLe13,JZL20}, and for deep representations \citep{SABE20}.  We note that our method does not rely on unrolling and is compatible with any network architecture as it acts directly on the weights of each layer. For the bulk of our experiments in \cref{sec:under}, we opted for unfolded networks because of the residual estimation at every layer: the input is approximately reconstructed at every layer to compute the next representation. As our work considers group actions as the centerpiece of group learning, having filters that act on the data space allows us to learn human-interpretable group actions, which becomes challenging when the group acts on a space of high-dimensional feature maps. At the same time, acting on the data space allows us to maintain moderate model sizes, making our presented networks efficient and scalable.

\textbf{Equivariance.} Consider an operator $f: V \to W$ and a family of actions $\mathcal{T}$. We call $f$ \emph{equivariant} with respect to the family $\mathcal{T}$ if for $T \in \mathcal{T}$ and any $x\in V$ it holds
\begin{equation}
  f(T(x)) = T'(f(x)),
\end{equation}
for some transformation $T' \in \mathcal{T}'$. Note that in general the action $T$ and transformation $T'$ are not the same; this can be directly seen since $\operatorname{dom}(T) = V$ and $\operatorname{dom}(T') = W$ (however, even when $\operatorname{dom}(T) = \operatorname{dom}(T')$ the operations need not be the same).

\textbf{Cyclic groups.} Consider a non-trivial set $G$ and an operation $*$. We call $(G, *)$ a \emph{group} if $*$ is associative on $G$, $G$ contains an identity element $e$ with respect to $*$, and for any $g \in G$ there exists an inverse element $g^{-1}\in G$ such that $g * g^{-1} = e$.

We call the group a cyclic group if every element in the group can be generated via consecutive applications of a basis element $g$ (i.e., any element of $G$ can be expressed as $g^k$ for $k\in \mathbb{Z}$), and we call $g$ the \emph{generator} of $G$. When the cyclic group is finite and has order $p$ it can be denoted as
\begin{equation}
  G = \{e, g, g^2, \ldots, g^{p-1}\},
\end{equation}
where higher order exponents get mapped to the $p$ canonical elements, i.e. for $l \geq p$ it holds $g^l = g^{l \mod p}$.

\textbf{Linear groups.} For our purposes, linear groups will refer to subgroups of the \emph{general linear group} $\operatorname{GL}_d(K)$ of $d\times d$ invertible matrices over the field $K$. Considering a collection of $C$ of elements of $\operatorname{GL}_d(K)$, the subgroup generated by $C$ is a linear group. A special case of interest for our work arises when $C$ contains only a single element $c$; then $C$ becomes a cyclic group generated by $c$.

\textbf{Unrolled sparse autoencoders.} Unrolled networks temporally unroll the steps of optimization algorithms, mapping algorithm iterations to network layers. In that way, the output of the neural network can be interpreted as the output of the optimization algorithm, with theoretical guarantees under certain assumptions. The \emph{Iterative Soft Thresholding Algorithn} (ISTA), an algorithm for sparse coding, has inspired several architectures \citep{GrLe10, SiEl19, SABE20, TDB21}, due to the desirability of sparse representations for interpretation purposes and also the connection between ReLU and soft thresholding. Within that framework, the representation at layer $l + 1$ is given by
\begin{equation}
    \label{eq:ista}
    \bm{z}^{(l+1)} = \mathcal{S}_{\lambda}\left(\bm{z}^{(l)} + \alpha\bm{W}_l^T(\bm{x} - \bm{W}_l\bm{z}^{(l)})\right),
\end{equation}
where $\bm{x}$ is the \emph{original} input, $\bm{z}^{(l)}$ is the representation at the previous layer, $\bm{W}_{l}$ are the weights of layer $l$, $\alpha$ is a constant such that $\frac{1}{\alpha} \geq \sigma_{\max}(\bm{W}_l^T\bm{W}_l)$, and $\mathcal{S}_{\lambda}$ is the \emph{soft thresholding} operator defined as
\begin{equation}
  \label{eq:soft_thresh}
    \mathcal{S}_{\lambda}(u) = \operatorname{sign}(u)\cdot\operatorname{ReLU}(\lvert u\rvert - \lambda).
\end{equation}
A one-sided version of \eqref{eq:soft_thresh} arises if we enforce $u > 0$. Then, the soft thresholding operator becomes a shifted version of ReLU, i.e. $\mathcal{S}_{\lambda}(u) = \operatorname{ReLU}(u - \lambda)$ (and the bias $\lambda$ can be incorporated to the weights of the network prior the activation). If the weights of all the $L$ layers are equal, i.e. $W_1 = \ldots = W_L$, we call the network \emph{tied}. As a final remark, \cref{eq:ista} can be rewritten as
\begin{equation*}
    \bm{z}^{(l+1)} = \mathcal{S}_{\lambda}\left((I - \alpha\bm{W}_l^T\bm{W}_l)\bm{z}^{(l)} + \alpha\bm{W}_l^T\bm{x}\right) = \mathcal{S}_{\lambda}\left(\bm{W}_z \bm{z}^{(l)} + \bm{W}_x \bm{x}\right),
\end{equation*}
where we let $\bm{W}_z = (I - \alpha\bm{W}_l^T\bm{W}_l)$ and $\bm{W}_x = \alpha\bm{W}_l^T$, which can be interpreted as a residual network \citep{HZRS16}, with a residual connection to the input. Convolutional extensions of \eqref{eq:ista} can be readily derived by replacing the matrix-vector products with correlation/convolution operations.

\section{Group learning framework}
\label{sec:framework}
We propose a framework for learning linear groups acting on the filters of neural networks. We consider cyclic linear groups generated by a single element of the generalized linear group; to that end, we will allow $K$ groups of size $p$ at every layer $l$ (all of which are design choices of the model). This implies that at every layer the architecture learns $K$ filter sets of exactly $p$ elements each, such that the filter sets are generated via a generating element $g_{(k,l)}$. Then, the weights of each group are related via the application of the group action
\begin{equation}
  [\bm{W}_l^k \quad g_{(k,l)} \bm{W}_l^k \quad \ldots g_{(k,l)}^{p-1} \bm{W}_l^k].
\end{equation}
Let $\bm{W}_l^k\in\mathbb{R}^{n\times m}$, i.e., the weights of every group at every layer are real-valued matrices of size $n\times m$. An initial approach would be to model $g_{(k,l)}$ as an element of $\operatorname{GL}_n(\mathbb{R})$, however we will see that this has significant limitations.

\textbf{Limitations of matrices in the weight space.} Consider using matrices $\bm{A}\in \mathbb{R}^{n \times n}$ to represent the generators $g_{(k,l)}$ of each group at every layer. Then, assuming a \emph{basis} filter $\bm{W}_l^k\in\mathbb{R}^{n\times m}$ for each filter set, generate the rest of the filters by applying the group action on the basis filter of each group
\begin{equation*}
  [\bm{W}_l^k \quad \bm{A}_{(k, l)} \bm{W}_l^k \quad \ldots \bm{A}_{(k, l)}^{p-1} \bm{W}_l^k].
\end{equation*}
This is a natural model since, if $\operatorname{rank}(A) = n$ (assuming $m \geq n$ for now), this linear operator can generate any vector in $\bm{x} \in \mathbb{R}^n$, as it is a basis for $\mathbb{R}^n$. However, while such an approach would be able to learn generators that can produce any vector in $\mathbb{R}^n$, it would ignore any spatial relations in the basis filters (or, more generally, any co-dependence or correlation between different columns). Indeed, matrices (and by extension, convolutional neural network filters) cannot be used to represent \emph{any} linear operator on their own set, as we prove by providing a short counter-example below.
\begin{proposition}
  There are linear operators on the space of $n \times m$ matrices that can't be represented via an $n \times n$ matrix.
\end{proposition}
\begin{proof}
  Assume that every linear operator on $n \times m$ matrices can be represented by a $n \times n$ matrix and consider an operator $f: V \to V$ (with $V \equiv \mathbb{R}^{n\times m}$) that swaps the top left with the bottom right element, i.e.
  \begin{equation*}
    \left(f(\bm{W})\right)_{ij} = \begin{cases}
      W_{nm}, \quad \text{ if $(i, j) = (1, 1)$},\\
      W_{11}, \quad \text{ if $(i, j) = (n, m)$},\\
      W_{ij}, \quad \text{ otherwise}.
    \end{cases}
  \end{equation*}
  This operator is linear since $f(\alpha \bm{X} + \beta \bm{Y}) = \alpha f(\bm{X}) + \beta f(\bm{Y})$. However, this operator can't be represented by a $n \times n$ matrix since these matrices, as linear operators, act on $n$-dimensional vectors (i.e., treat every column of the input independently). This follows since, assuming $f$ is parametrized by $\bm{A}\in\mathbb{R}^{n\times n}$, we have
  \begin{equation*}
    \left(f(\bm{W})\right)_{ij} = \sum_{k} A_{ik}W_{kj}.
  \end{equation*}
  In the above equation the operator has a strict dependence on column $j$ of the input and the computation is independent of all other columns; therefore, swapping the corner most elements columns $1$ and $m$ is not feasible using this parametrization.

  Note that a simpler example would be an operator such that $\bm{W} \mapsto \bm{W}^T$, as transposition is a linear map. However, if $n \neq m$ then the group elements would belong in different linear spaces ($\mathbb{R}^{n\times m}$ versus $\mathbb{R}^{m\times n}$). While this is not prohibitive, it would complicate the notation for the group definition and hence we presented a slightly more involved example.
\end{proof}

\textbf{Employing vectorization.} As we showed above, modeling the group actions $g_{(k,l)}$ as $n\times m$ matrices leads to restrictive groups where simple linear operators are excluded. Instead, we propose the modeling of the group actions $g_{(k,l)}$ via vectorization. We first define the vectorization operator along with its inverse.
\begin{definition}
  \label{def:vec}
  Consider a matrix $\bm{A} \in \mathbb{R}^{n\times m}$. Define the vectorization operator, denoted as $\operatorname{vec}$, as follows
  \begin{align}
    \label{eq:vec}
    \begin{split}
      \operatorname{vec}: \quad& \mathbb{R}^{n\times m} \to \mathbb{R}^{n\cdot m}\\
      & \bm{A} \mapsto [A_{11}, \ldots, A_{n1}, A_{12}, \ldots, A_{n2},\ldots, A_{1m}, \ldots, A_{nm}]^T.
    \end{split}
  \end{align}
  The vectorization operator can also be expressed as a linear sum using the \emph{Kronecker product}: $\operatorname{vec}(\bm{A}) = \sum_{i = 1}^m \bm{e}_i \otimes \bm{A}\bm{e}_i$, where $\bm{e}_i\in\mathbb{R}^m$ denotes the $i$-th basis vector of $\mathbb{R}^m$. The inverse map $\operatorname{vec}^{-1}_{n\times m}$, where the subscript $n\times m$ will be dropped for conciseness, is also defined via Kronecker products
  \begin{align}
    \label{eq:vec_inv}
    \begin{split}
      \operatorname{vec}^{-1}: \quad& \mathbb{R}^{n \cdot m} \to \mathbb{R}^{n\times m}\\
      & \bm{a} \mapsto \left(\operatorname{vec}^T(\bm{I}_m)\otimes \bm{I}_n\right)(\bm{I}_m\otimes\bm{a}),
    \end{split}
  \end{align}
  where $\bm{I}_n \in \mathbb{R}^{n\times n}$ denotes the identity matrix of $\mathbb{R}^n$.
\end{definition}
Using those definitions, we will parametrize each group action of every layer as a linear operator on vectors in $\mathbb{R}^{n\cdot m}$. Any linear operator on vectors can be uniquely parametrized (up to similarity) via a matrix $\bm{A} \in \mathbb{R}^{n\cdot m \times n\cdot m}$; then, the group action acting on each filter set $k$ at every layer $l$,  $g_{(k,l)}$, will be instantiated via the map
\begin{align}
  \label{eq:group}
  \begin{split}
    \phi_{\bm{A}}: \quad& \mathbb{R}^{n\times m} \to \mathbb{R}^{n\times m}\\
    & \bm{X} \mapsto \operatorname{vec}^{-1}(\bm{A}\operatorname{vec}(\bm{X})).
  \end{split}
\end{align}
Using the formulation of \eqref{eq:group}, and denoting consecutive compositions with the same function via $f^n = f \circ f^{n-1}$ with $f \circ f = f^2$, we can finally express the weights of each filter set as a function of the group action $\bm{A}_{(k, l)}\in \mathbb{R}^{n\cdot m \times n \cdot m}$
\begin{equation}
  \label{eq:group_layer}
  \bm{W}_{l_k} = [\bm{W}_l^k \quad \phi_{\bm{A}_{(k, l)}}(\bm{W}_l^k) \quad \ldots \phi_{\bm{A}_{(k, l)}}^{p-1}(\bm{W}_l^k)].
\end{equation}
While simple in its inception, the operator of \eqref{eq:group} is expressive and can model \emph{all} linear operators in its domain. Note that $\operatorname{dom}(\phi_{\bm{A}_{(k, l)}}) = \mathbb{R}^{n\times m}$ which is different from the domain of $\bm{A}_{(k, l)}$ ($=\mathbb{R}^{n\cdot m \times n \cdot m}$) when viewed as an operator (which is trivially linear with respect to its domain).
\begin{proposition}[Linearity]
  Let $\phi_{\bm{A}_{(k, l)}}$ be the map defined in \eqref{eq:group}; $\phi_{\bm{A}_{(k, l)}}$ is \emph{linear}. Moreover, \emph{any} linear map $\phi: \mathbb{R}^{n\times m} \to \mathbb{R}^{n\times m}$ can be parametrized by $\phi_{\bm{A}_{(k, l)}}$.
\end{proposition}
\begin{proof}
  $\phi_{\bm{A}_{(k, l)}}$ is a composition of linear maps and therefore is itself linear. Indeed, $\phi_{\bm{A}_{(k, l)}}$ comprises the composition of $\operatorname{vec}$, matrix multiplication, and $\operatorname{vec}^{-1}$. Following \eqref{eq:vec} and \eqref{eq:vec_inv} in \cref{def:vec}, both operators are linear as compositions of linear operations, and therefore $\phi_{\bm{A}_{(k, l)}}$ is itself linear.

  For the second part, consider the vector spaces $\mathbb{R}^{n\times m}$ and $\mathbb{R}^{n\cdot m}$. The linear map $\operatorname{vec}$ is an isomorphism from $\mathbb{R}^{n\times m}$ to $\mathbb{R}^{n\cdot m}$ since
  \begin{itemize}
    \item $\operatorname{vec}(\alpha \bm{W}) = \alpha \operatorname{vec}(\bm{W})$, for any $\alpha\in\mathbb{R}, \bm{W} \in \mathbb{R}^{n\times m}$, and
    \item $\operatorname{vec}(\bm{W} + \bm{V}) = \operatorname{vec}(\bm{W}) + \operatorname{vec}(\bm{V})$, for any $\bm{W}, \bm{V} \in \mathbb{R}^{n\times m}$,
  \end{itemize}
  with the inverse map $\operatorname{vec}^{-1}$. Since the spaces are isomorphic, parametrizing the a linear map in $\mathbb{R}^{n\times m}$ reduces to parametrizing a linear map in $\mathbb{R}^{n\cdot m}$. However, all linear transformations in $\mathbb{R}^{n\cdot m}$ can be expressed by matrices $\bm{A}\in\mathbb{R}^{n\cdot m \times n\cdot m}$\footnote{For a proof of this classical result in linear algebra, see \cref{appendix:linalg}.}, which is precisely the parametrization of $\phi_{\bm{A}_{(k, l)}}$. Therefore, any linear map $\phi$ can be parametrized by  $\phi_{\bm{A}_{(k, l)}}$.
\end{proof}

\textbf{Architecture.} Composing the contents of this section, we apply our method to an unfolded network. Our building block consists of a \emph{cyclical group layer}, a convolutional layer which utilizes unfolding
\begin{equation}
  \bm{z}^{(l+1)} = \mathcal{S}_{\lambda}\left(\bm{z}^{(l)} + \alpha\bm{W}_l^T*(\bm{x} - \bm{W}_l\bm * {z}^{(l)})\right),
\end{equation}
where $*$ denotes convolution (correlation) and the weights of each layer $l$ have $K$ groups of filter sets such that
\begin{equation}
  \label{eq:layer_weights}
  \bm{W}_l = [\bm{W}_{l_1} \quad \bm{W}_{l_2} \quad \ldots \quad \bm{W}_{l_K}],
\end{equation}
with $\bm{W}_{l_k}$ being defined by \eqref{eq:group_layer}.

\subsection{Invertibility loss}
\label{subsec:invert}
To train the architecture, we consider the loss function best applicable to the downstream task of interest, which would result in the simultaneous, unsupervised learning of the basis filters $\bm{W}_{l}^k$ and the group action $\bm{A}_{(k, l)}$ via backpropagation. However, without any regularization, $\bm{A}_{(k, l)}$ are unlikely to be elements of the linear group $\operatorname{GL}_{n\cdot m}(\mathbb{R})$. The group membership is defining in our work, as otherwise $\bm{A}_{(k,l)}$ do not define a cyclic group. To that end, we introduce an \emph{invertibility loss}. This loss encourages the elements of the cyclic groups to be invertible and thus are elements of $\operatorname{GL}_{n\cdot m}(\mathbb{R})$:
\begin{equation}
  L = \mu \lVert \bm{A}_{(k, l)} \widetilde{\bm{A}}_{(k, l)} - \bm{I}\rVert_F,
\end{equation}
where $\widetilde{\bm{A}}_{(k, l)}$ are matrices only used in training to encourage invertibility and $\mu$ is the a regularization parameter controlling the tradeoff between the performance on the downstream task and the enforcement of the invertibility. More discussions about the invertibility loss, and alternatives, can be found in \cref{appendix:loss}.

\section{Experiments}
\label{sec:under}
We used the architecture introduced in \cref{sec:framework} in order to learn filter sets governed by different group actions in order to uncover latent symmetries in natural datasets. For our experimental setting, we learned basis filters $\bm{W}_l^k \in\mathbb{R}^{6 \times 6}$. Our architecture consists of $L = 4$ layers, each having $K = 5$ cyclic groups of $p = 4$ elements each. For complete information about hyperparameter values, datasets, training procedures, and the architecture see \cref{appendix:architecture}.

\subsection{Recovered group structures}
\label{subsec:recov}
We trained an unfolded network using our method introduced in \cref{sec:framework} on the \texttt{CIFAR10} dataset for the task of classification and learned the group actions $\bm{A}_{(k, l)}$ and the basis filters $\bm{W}_l^k$. During our experiments our networks learned multiple interesting group actions; we note three main structures that are of interest and we present them in \cref{fig:group_actions}.

\begin{figure}[b]
  \centering
  \begin{subfigure}[h]{0.32\textwidth}
    \centering
    \includegraphics[width=\textwidth]{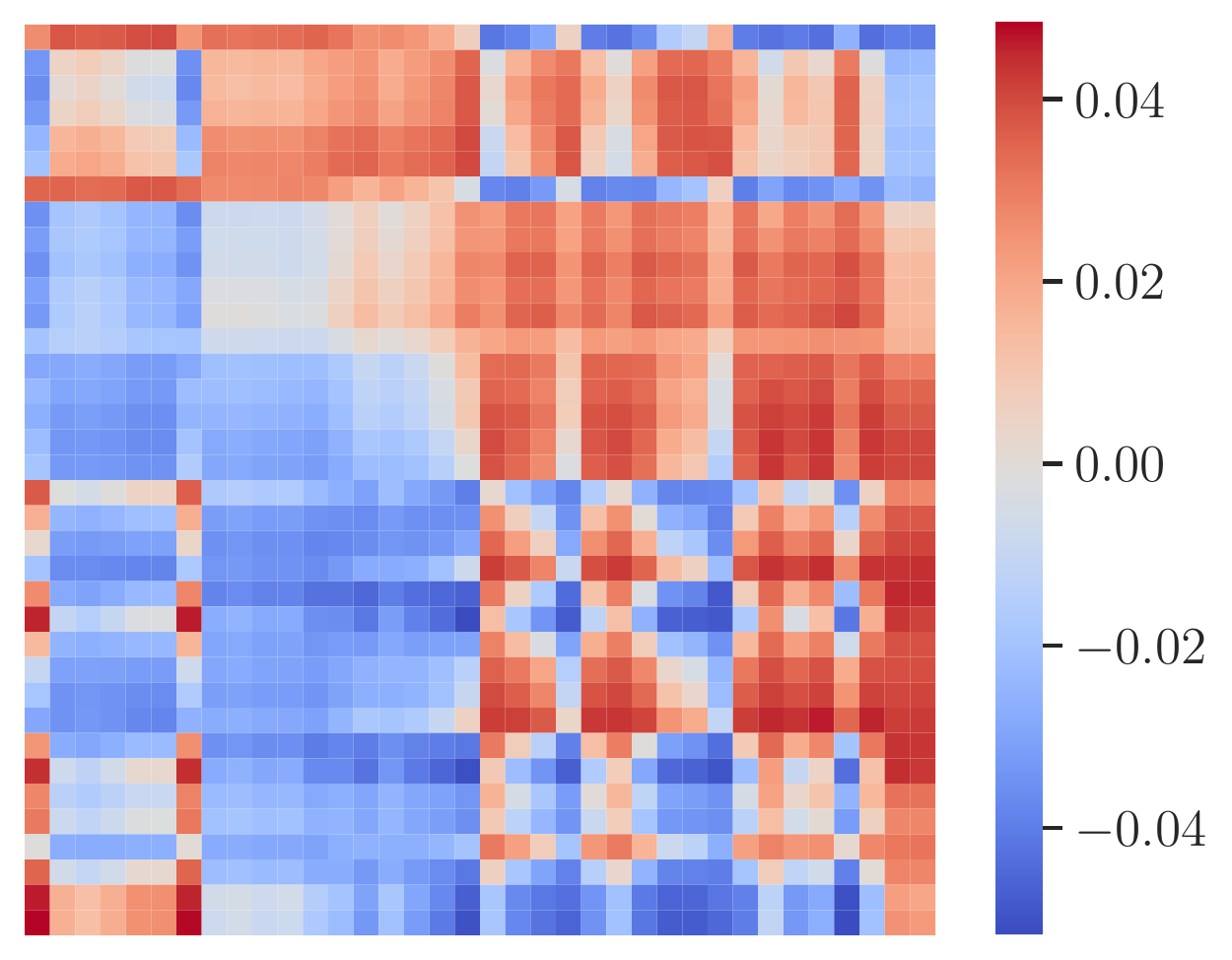}
  \end{subfigure}
  \begin{subfigure}[h]{0.32\textwidth}
    \centering
    \includegraphics[width=\textwidth]{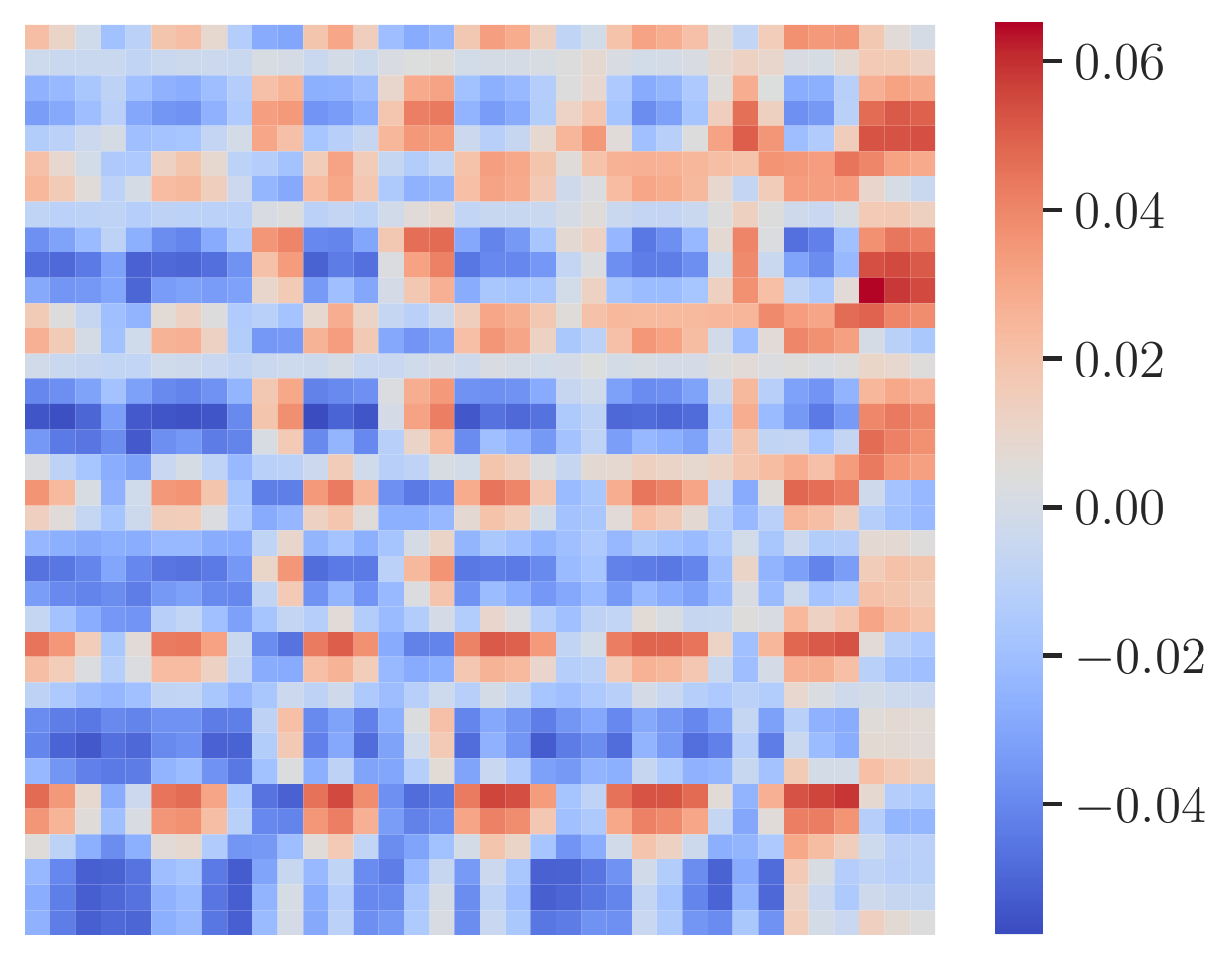}
  \end{subfigure}
  \begin{subfigure}[h]{0.32\textwidth}
    \centering
    \includegraphics[width=\textwidth]{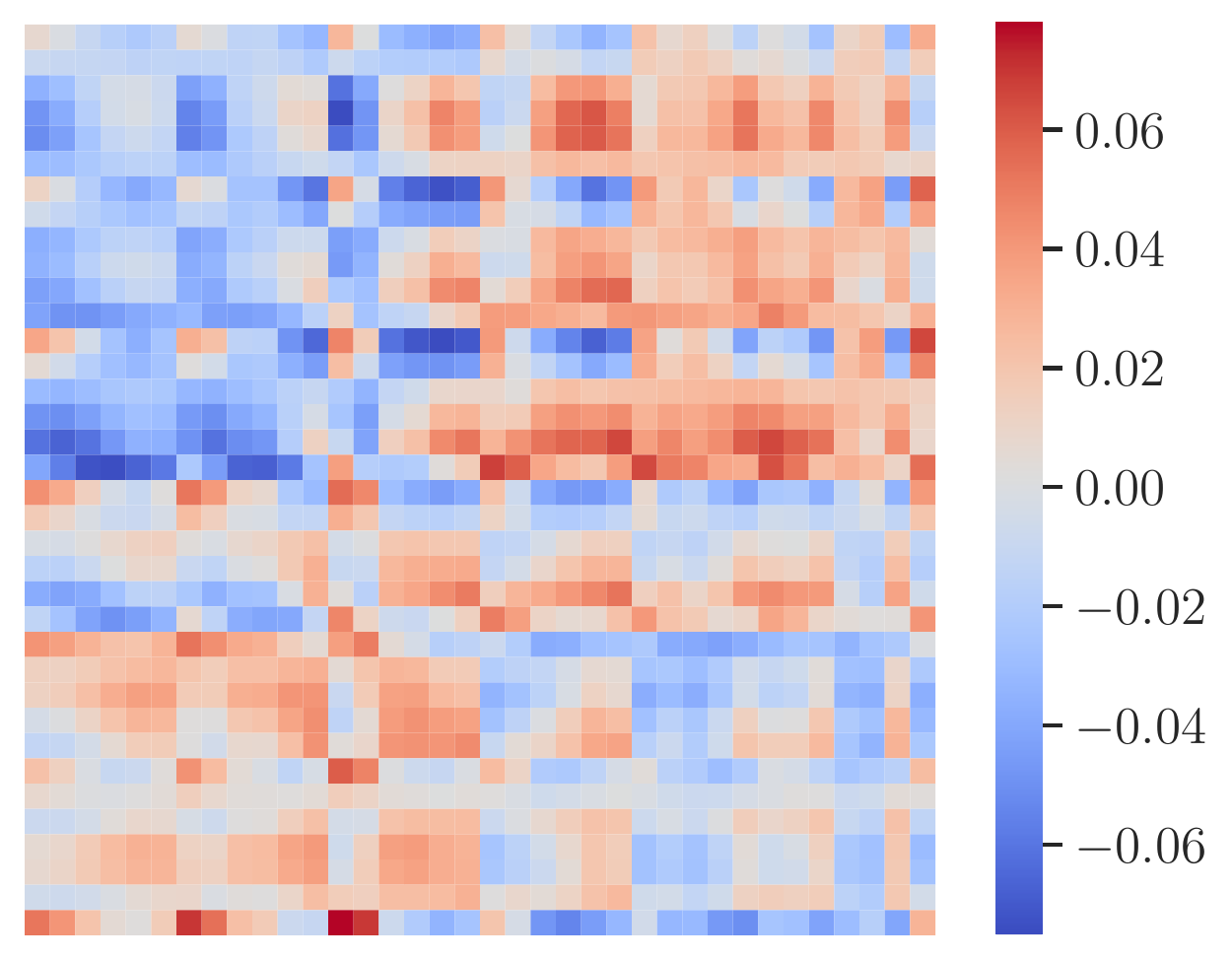}
  \end{subfigure}

  \begin{subfigure}[h]{0.32\textwidth}
    \centering
    \includegraphics[width=\textwidth]{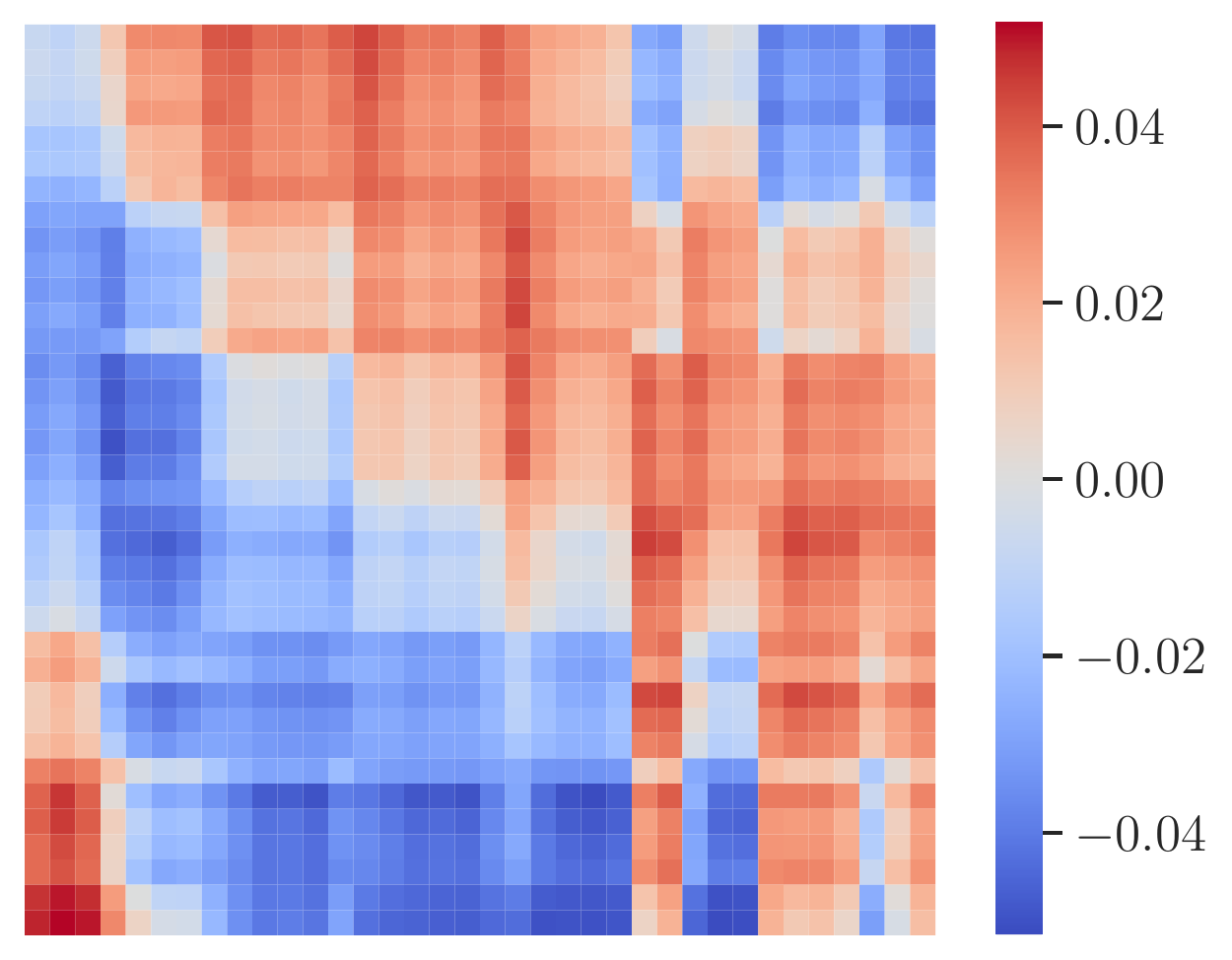}
    \caption{Skew-symmetric}
  \end{subfigure}
  \begin{subfigure}[h]{0.32\textwidth}
    \centering
    \includegraphics[width=\textwidth]{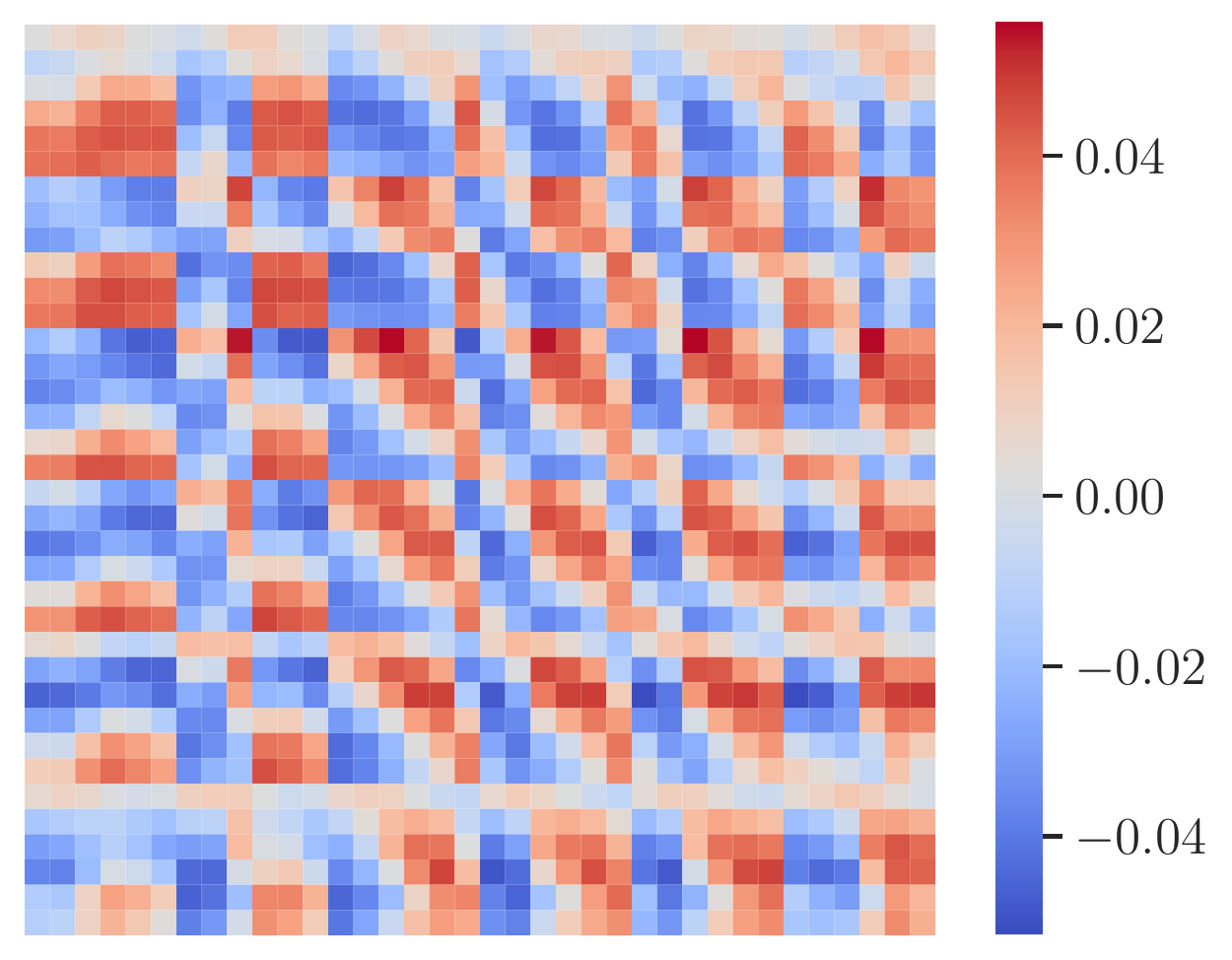}
    \caption{Toeplitz}
    \label{fig:toeplitz}
  \end{subfigure}
  \begin{subfigure}[h]{0.32\textwidth}
    \centering
    \includegraphics[width=\textwidth]{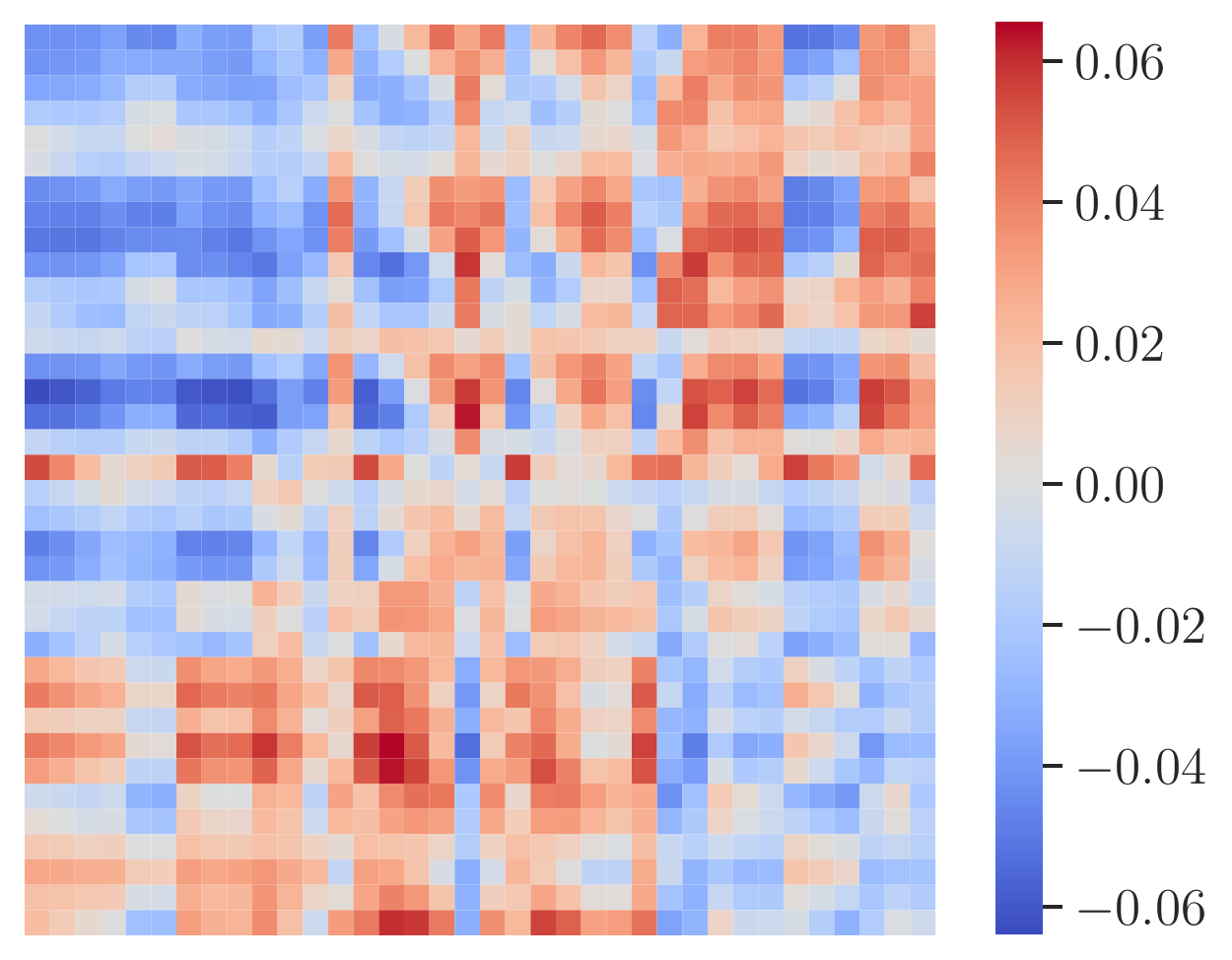}
    \caption{Multi-scale}
  \end{subfigure}
  \caption{Emerging group structures when learning group actions for classification on \texttt{CIFAR10}.}
  \label{fig:group_actions}
\end{figure}

\textbf{Skew-symmetric structure.} Observing the first column of \cref{fig:group_actions} we notice a skew-symmetric structure. Considering a single row of these matrices, most of them implement a \emph{causal averaging}: the value of each pixel is updated according to a weighted sum of the pixels following it. This is an interesting emerging structure as it corresponds broadly to two very well known operations: \emph{filtering} (or smoothing) from Computer Vision and \emph{average pooling} in deep learning.

\textbf{Toeplitz structure.} In the second column, the group actions have a Toeplitz (or approximately circulant) structure. Circulant matrices are of particular interest because of their ties to convolution, indicating a scheme of weight sharing or permutation operation being performed. Moreover, they are diagonalized by the Discrete Fourier Transform, which we explore in \cref{appendix:fourier}.

\begin{wrapfigure}{r}{0.62\textwidth}
  \centering
  \begin{subfigure}[h]{0.2\textwidth}
    \centering
    \includegraphics[width=\textwidth]{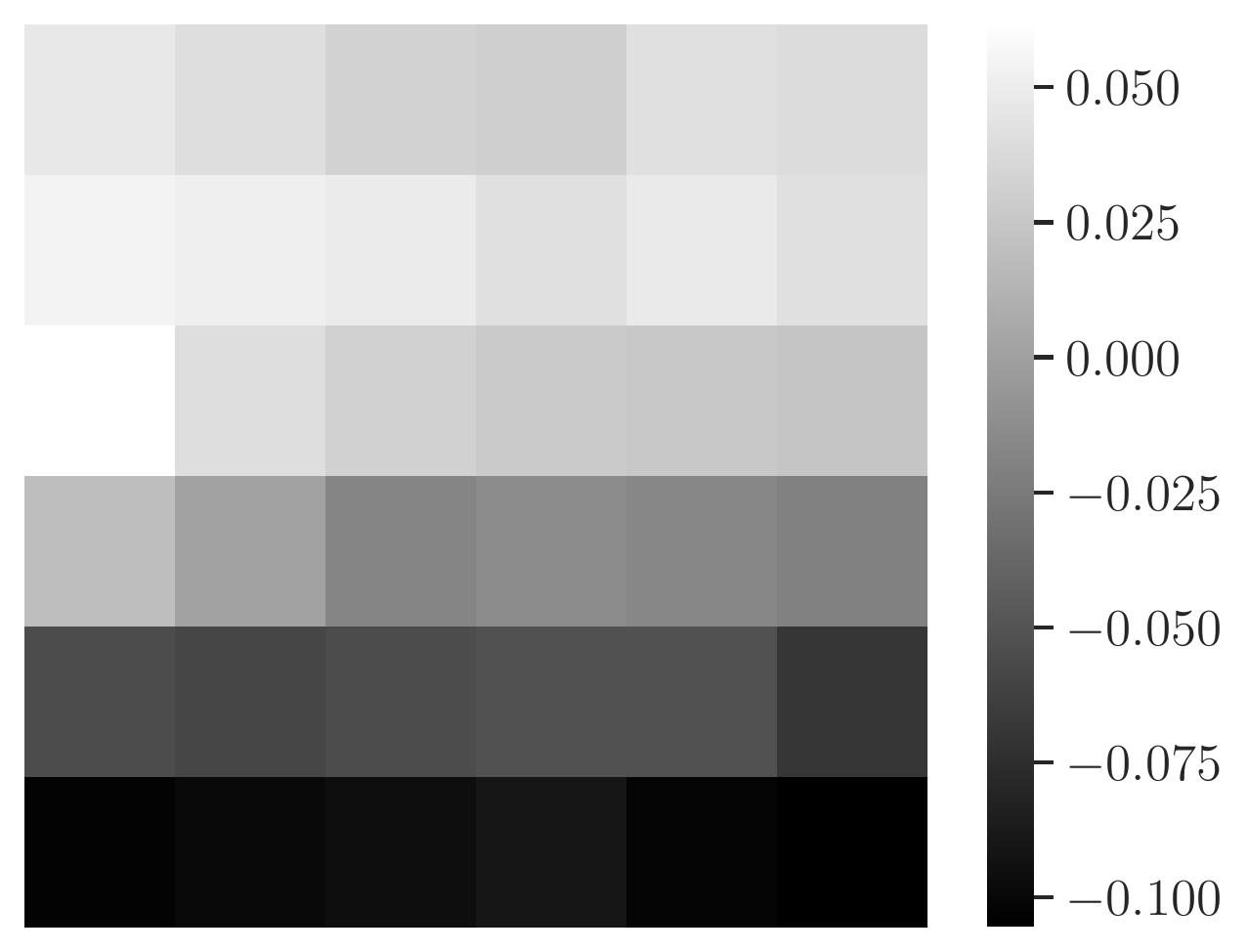}
    \caption{Skew-symmetric}
  \end{subfigure}
  \begin{subfigure}[h]{0.2\textwidth}
    \centering
    \includegraphics[width=\textwidth]{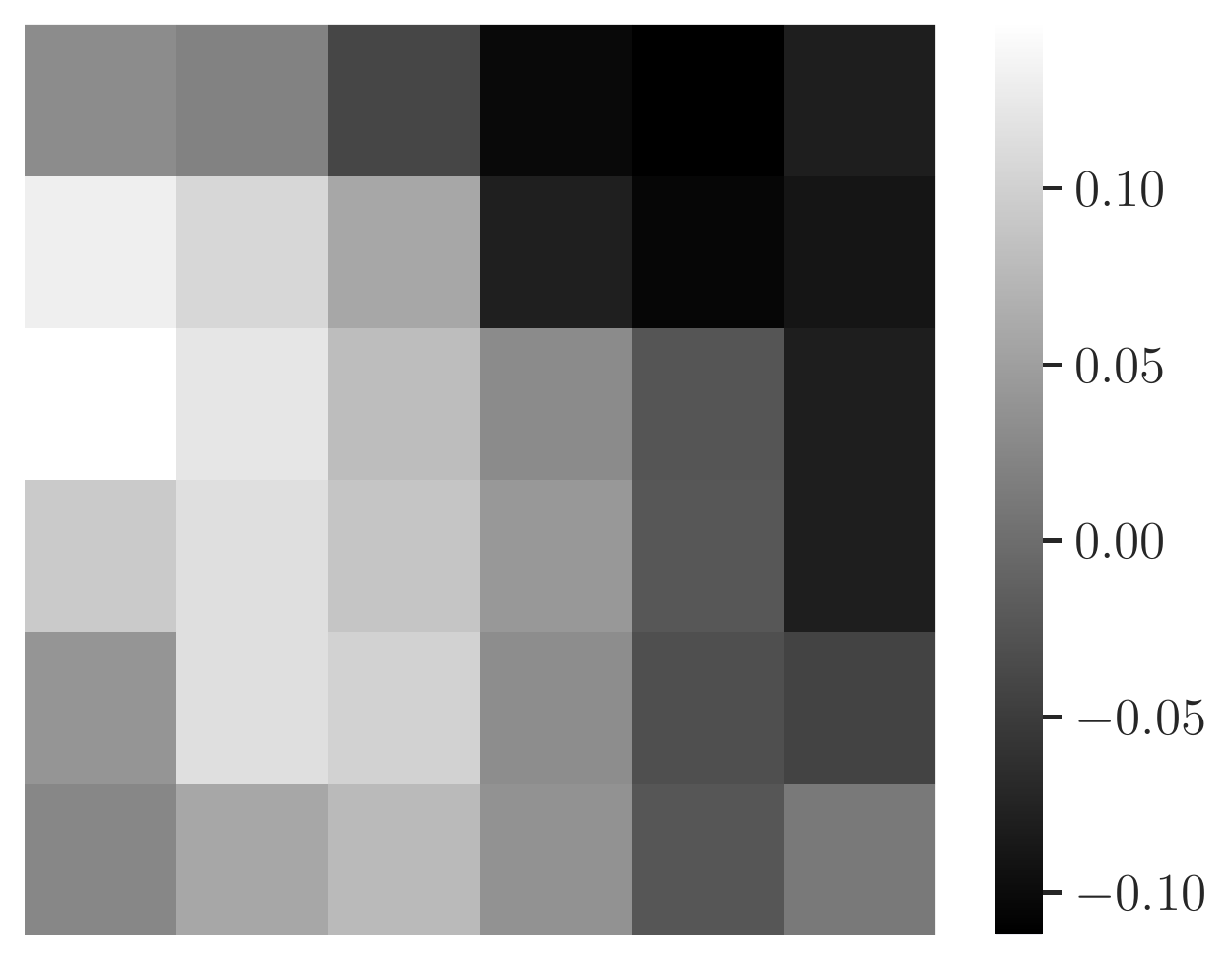}
    \caption{Toeplitz}
  \end{subfigure}
  \begin{subfigure}[h]{0.2\textwidth}
    \centering
    \includegraphics[width=\textwidth]{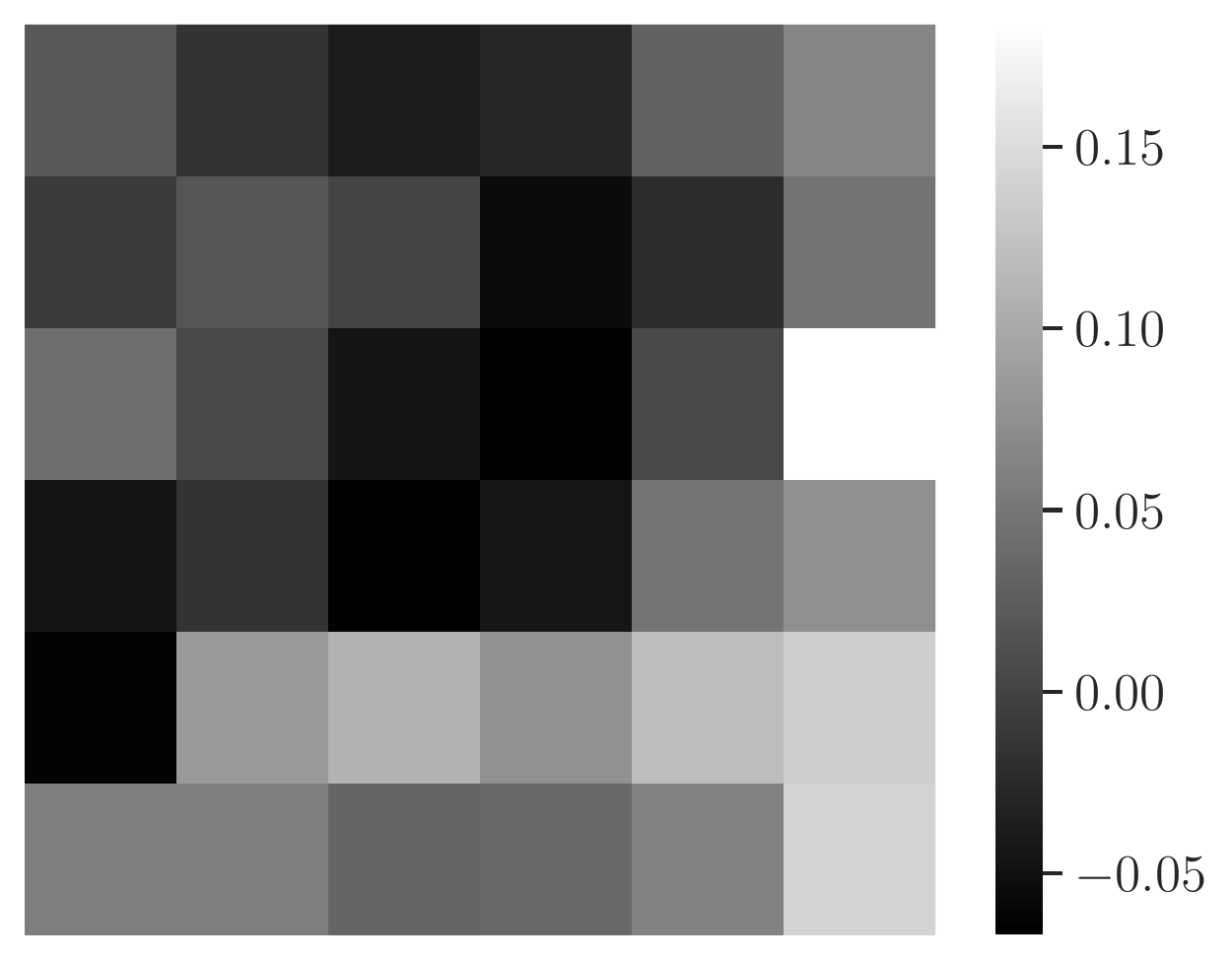}
    \caption{Multi-scale}
  \end{subfigure}
  \caption{Effect of the group actions of \cref{fig:group_actions} on identity matrices.}
  \label{fig:action_effect}
\end{wrapfigure}

\textbf{Multi-scale structure.} Finally, the group actions in the third column have a multi-scale structure. Indeed, examining the recovered matrices at the quadrant level, we observe each quadrant having a dominant sign signature, either positive or negative (for example, in both figures, the lower left quadrant contains mostly positive values). However, ``zooming'' into specific blocks that describe the relation between rows and columns of the original filter ($6\times 6$ blocks) we observe significantly different structures, both within quadrants (for example, the last and third to last block in the first ``block-row'' of the bottom figure) and across quadrants (the first and last blocks in the same ``block-row'').

\subsubsection{Effect of the group actions}
To gain better intuition and evaluate our analysis of the recovered structures, we applied the group actions to $\bm{I}_6 \in \mathbb{R}^{6\times 6}$, the identity matrix of $\mathbb{R}^6$. Due to its simple structure, it allows us to understand the effect of the linear operations defined by $\bm{A}_{(k, l)}$. Note that this is not trivial, as $\bm{A}_{(k, l)}$ acts on the vectorization of $\bm{I}_6$. We applied the group actions of the bottom row of \cref{fig:group_actions} on $\bm{I}_6$ and we visualize the effects of the group actions in \cref{fig:action_effect}.

The effect of the skew-symmetric operator aligns with our intuition from \cref{subsec:recov}: the earlier pixels have higher intensities compared to the later ones. This is expected, as the average is over more pixels in the upper rows of the relevant group actions. This translates to more elements of the diagonal of $\bm{I}_6$ being included in the average, resulting in this gradient-like transformation on the input. For the Toeplitz structure, the diagonal has been slightly rotated towards the lower-left part of the matrix and has been significantly smoothed out. As we will further argue in \cref{subsubsec:comp}, this behavior is not surprising as the structure of the group actions resembles a composition of known operators. Finally, the multi-scale structure is harder to interpret; however, we note that considering the pixels above the antidiagonal we observe significantly lower values compared to the pixels below the antidiagonal. Examining the upper or lower parts in a different scale, we see different structures (positive values in the upper part and negative values in the lower part), in line with the multi-scale interpretation.

\subsection{Group structures in the wild}
To further interpret the recovered groups of \cref{subsec:recov}, we created synthetic experiments where we control the group action that is being performed. To that end, we considered the \texttt{CIFAR10} dataset and we designed the following experiment
\begin{itemize}[noitemsep,nolistsep]
  \item For every image $i$ in the dataset, we extract a random patch $\bm{x}_{p_i}$ of size $6\times 6$.
  \item We apply a transformation on the extracted patch to get $\bm{y}_{p_i}$ and consider the tuple $(\bm{x}_{p_i}, \bm{y}_{p_i})$ as a training point. The transformations we apply are rotations over a fixed angle and average pooling of a fixed radius\footnote{We used a variation \href{https://gist.github.com/rwightman/f2d3849281624be7c0f11c85c87c1598}{Ross Wightman}'s implementation of median filtering for PyTorch.}.
  \item We train a single layer linear network with no nonlinearity such that $\hat{\bm{y}} = \bm{A}\bm{x}$ and train using the MSE loss.
\end{itemize}
The above experimental setup simulates the application of the group action on the filter groups. By successfully learning $\bm{A}$ in this sterile setting where we control the action between successive elements we can interpret the recovered actions of \cref{subsec:recov}.

The results are presented in \cref{fig:median,fig:rot} for average pooling and rotating the patches, respectively. We observe that the learned matrices have the ``expected'' structure: rotation matrices of $90\degree$ have a familiar grid structure, and the structure of the interpolations aligns with that reported, for example, in \citep{DWL+21}. For average pooling, we again see a diagonal structure that respects the interactions between the pixels. However, we see that both structures differ from the exact observed structure of \cref{fig:toeplitz}, which we discuss below.

\begin{figure}[t]
  \centering
  \begin{subfigure}[h]{0.24\textwidth}
    \centering
    \includegraphics[width=\textwidth]{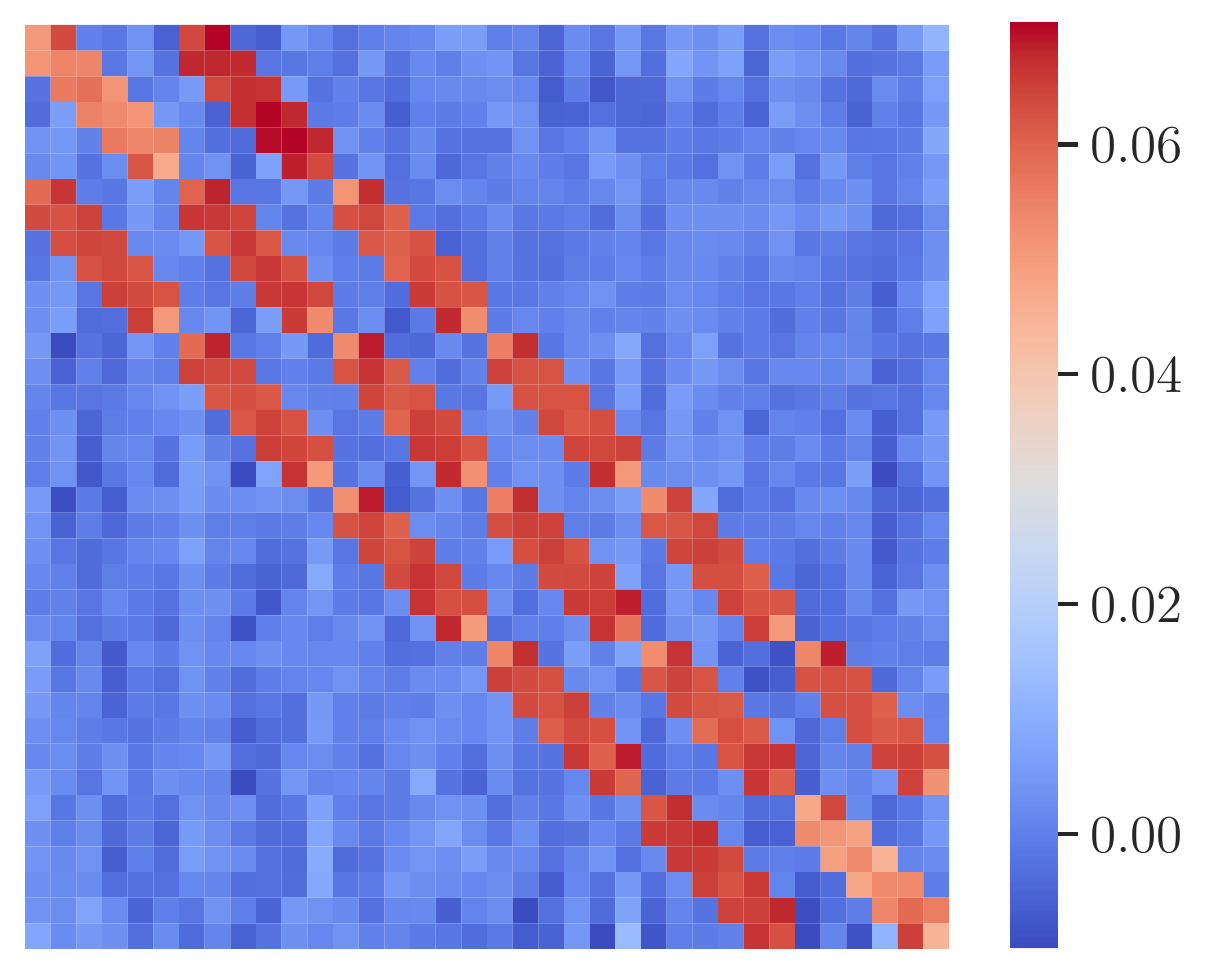}
    \caption{$r = 3$}
  \end{subfigure}
  \begin{subfigure}[h]{0.24\textwidth}
    \centering
    \includegraphics[width=\textwidth]{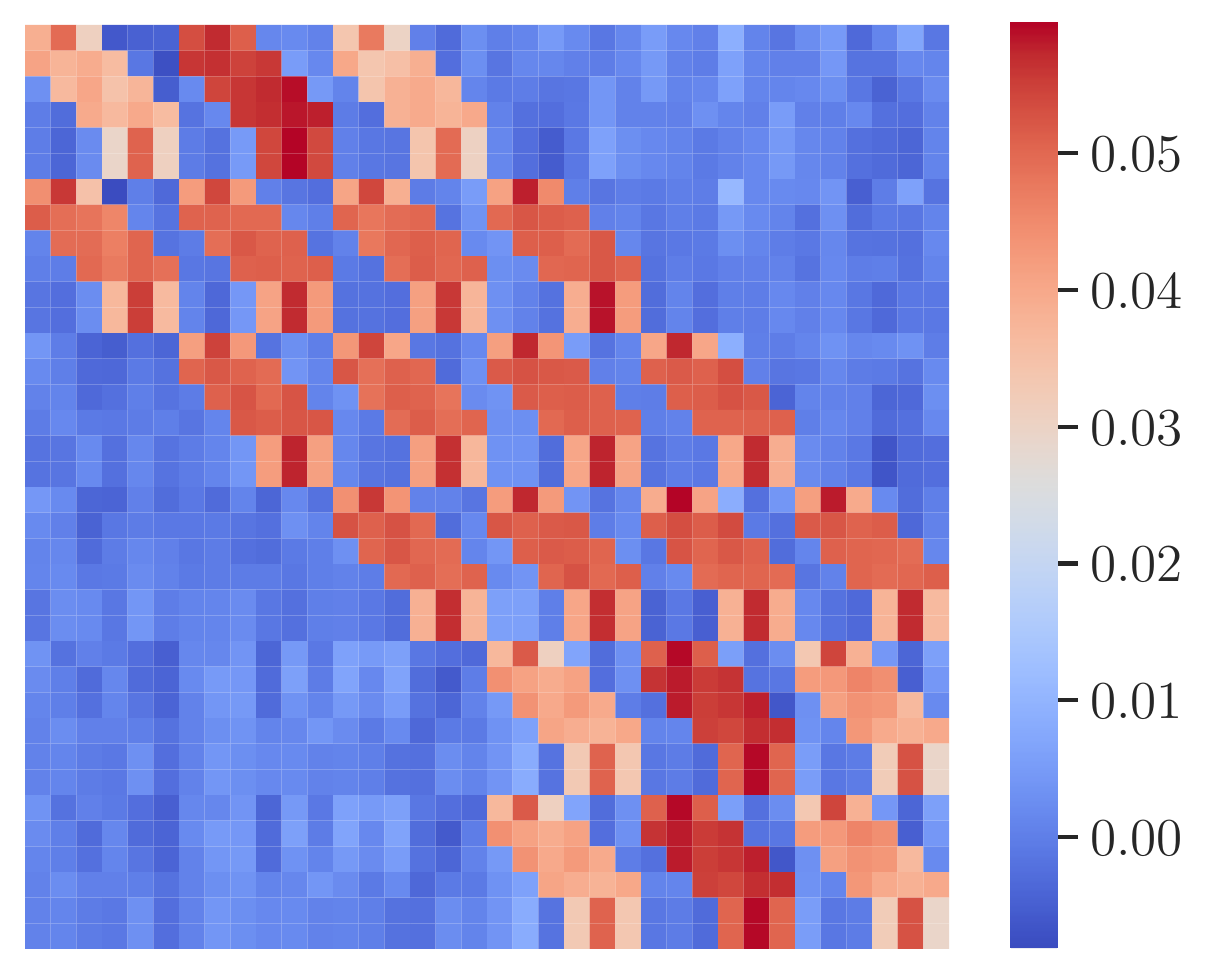}
    \caption{$r = 4$}
  \end{subfigure}
  \begin{subfigure}[h]{0.24\textwidth}
    \centering
    \includegraphics[width=\textwidth]{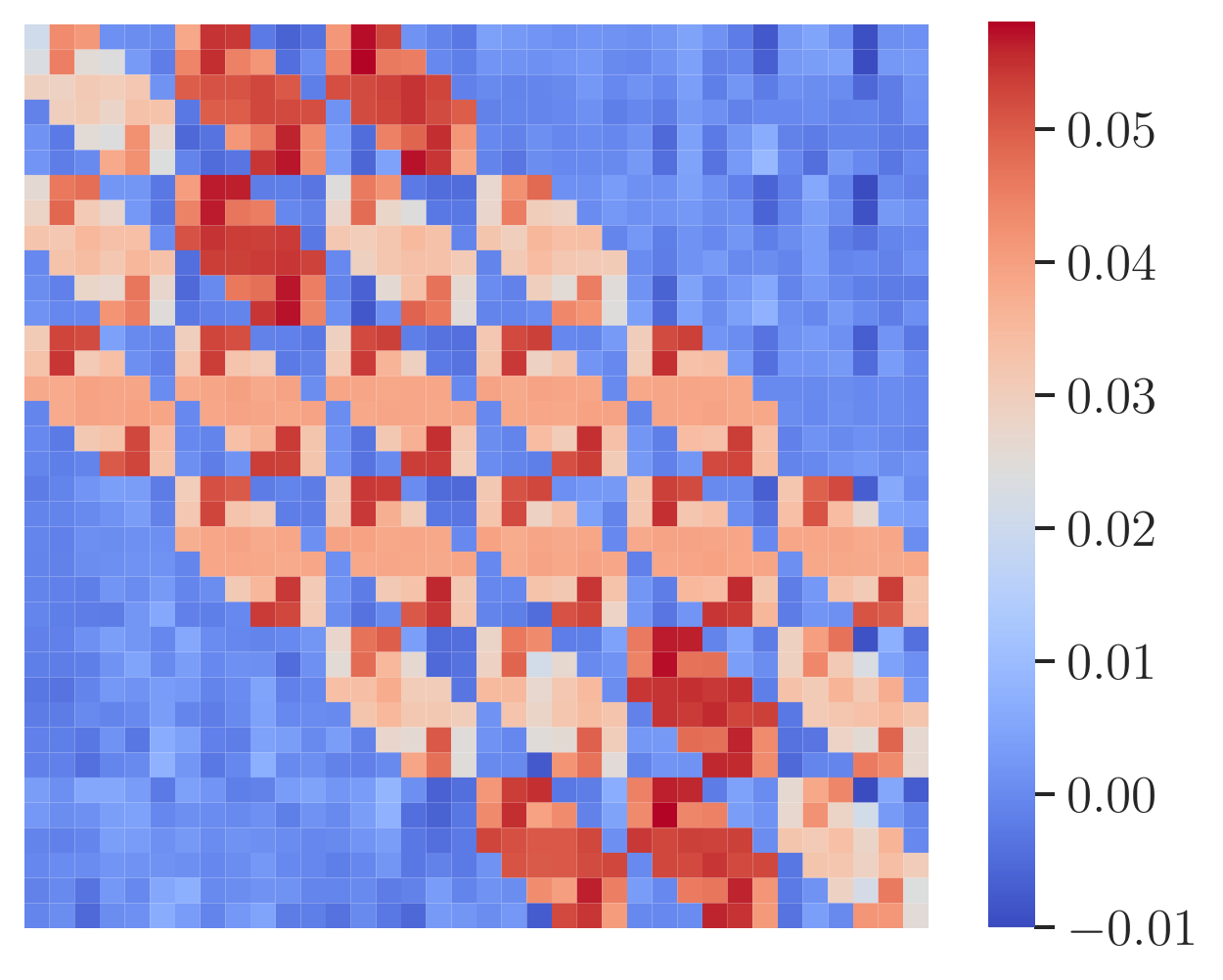}
    \caption{$r = 5$}
  \end{subfigure}
  \begin{subfigure}[h]{0.24\textwidth}
    \centering
    \includegraphics[width=\textwidth]{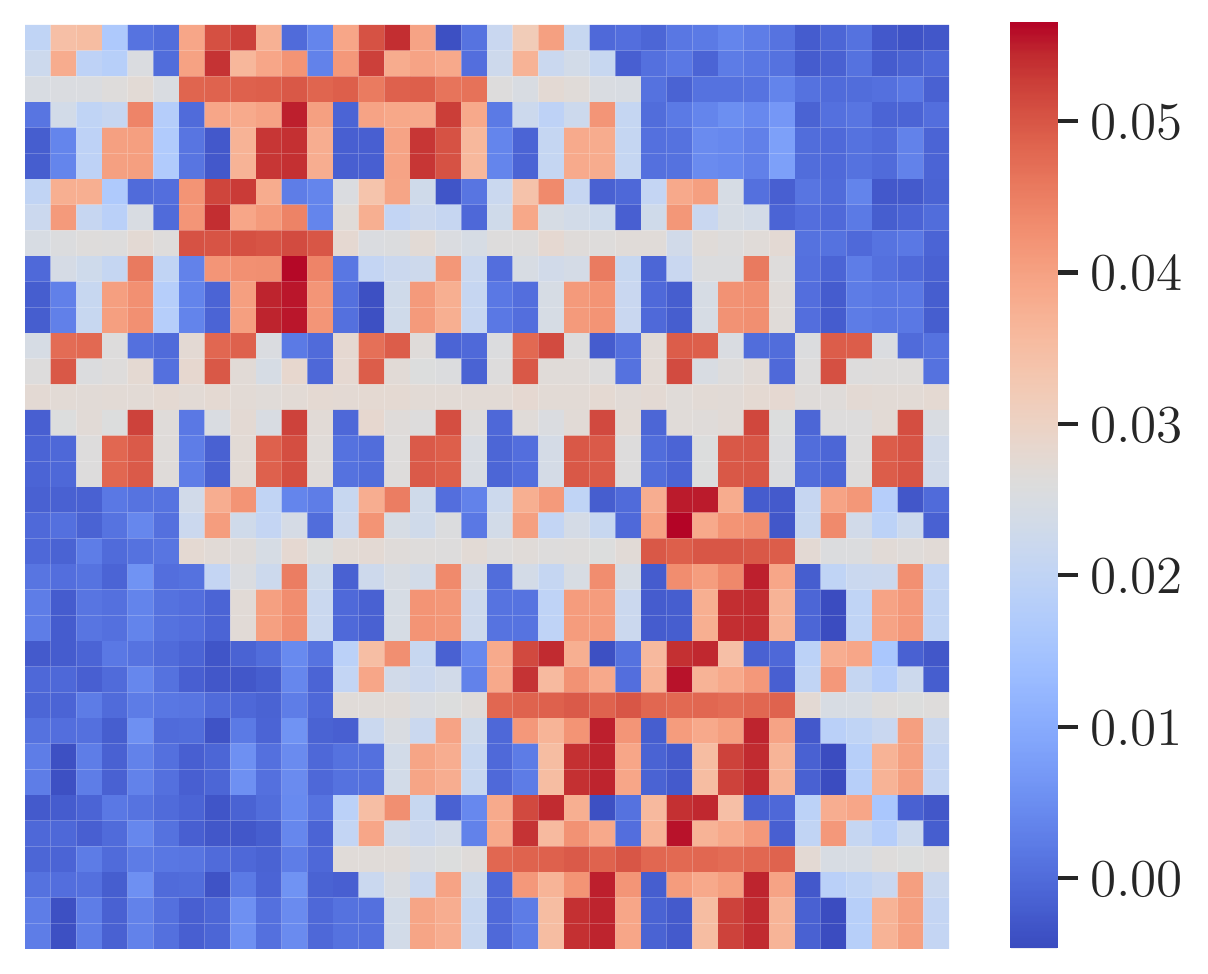}
    \caption{$r = 6$}
  \end{subfigure}
  \caption{Learned group actions when performing average pooling using a radius $r \in \{3, 4, 5, 6\}$.}
  \label{fig:median}
\end{figure}

\begin{figure}[t]
  \centering
  \begin{subfigure}[h]{0.24\textwidth}
    \centering
    \includegraphics[width=\textwidth]{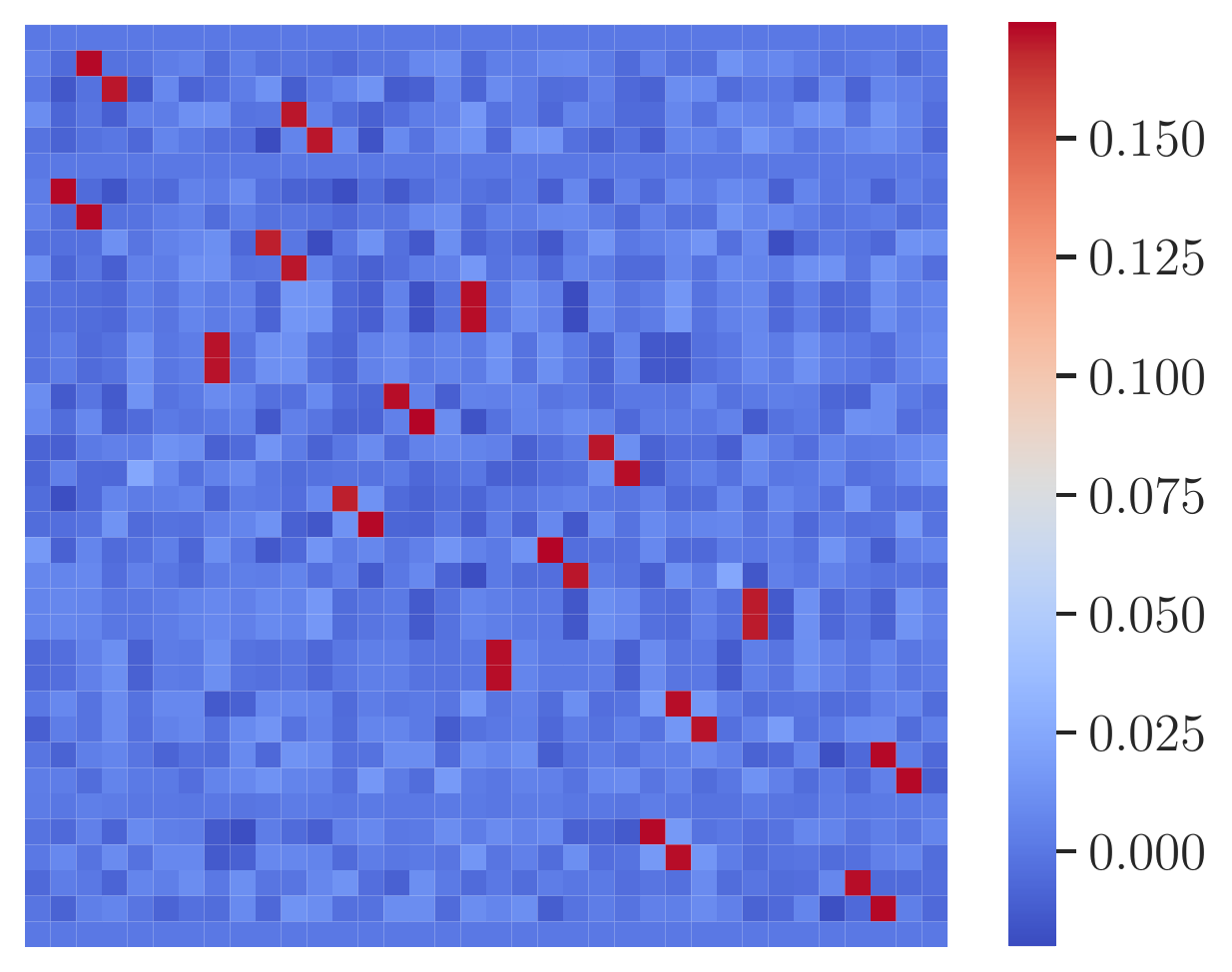}
    \caption{$\theta = 30\degree$}
  \end{subfigure}
  \begin{subfigure}[h]{0.24\textwidth}
    \centering
    \includegraphics[width=\textwidth]{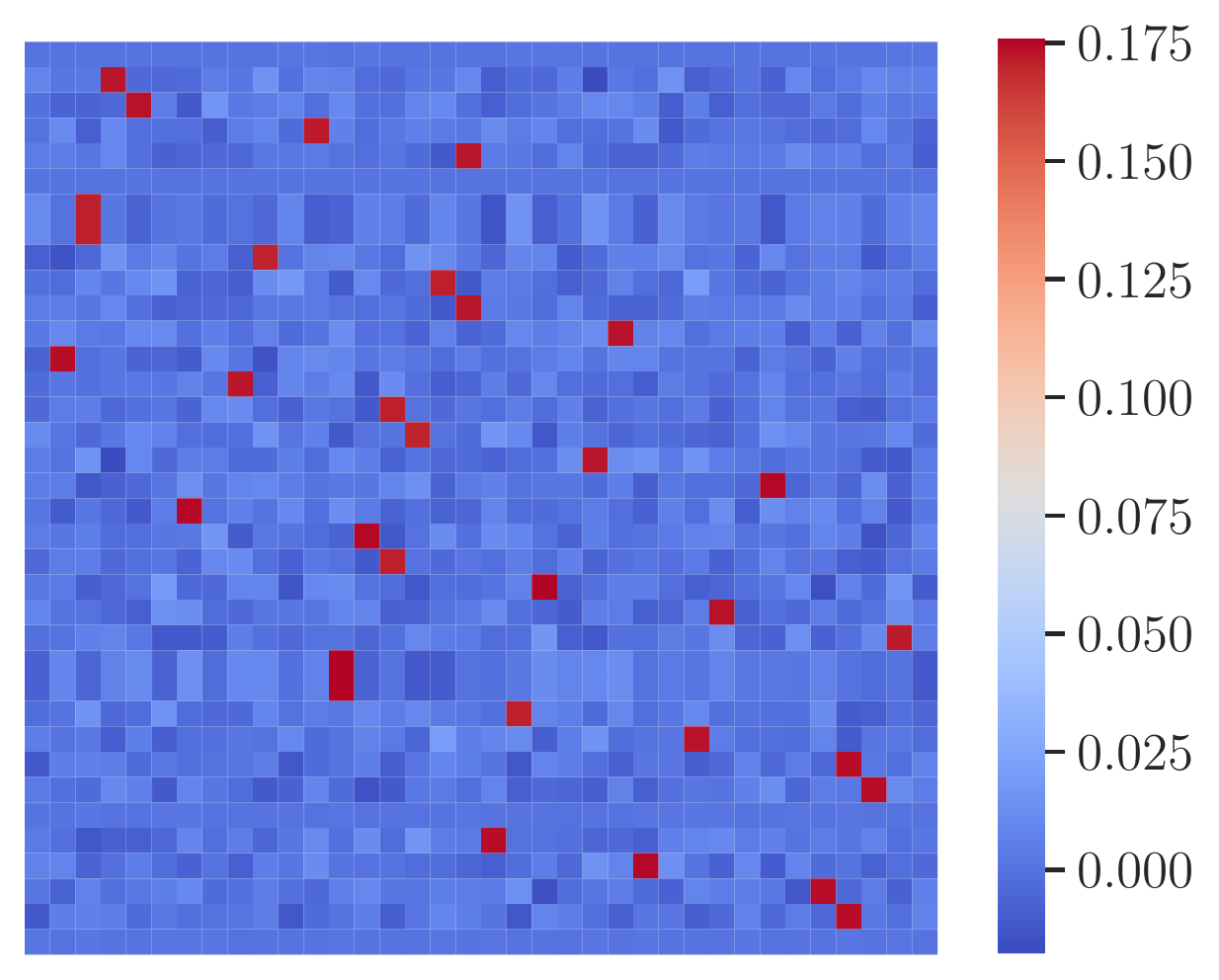}
    \caption{$\theta = 45\degree$}
  \end{subfigure}
  \begin{subfigure}[h]{0.24\textwidth}
    \centering
    \includegraphics[width=\textwidth]{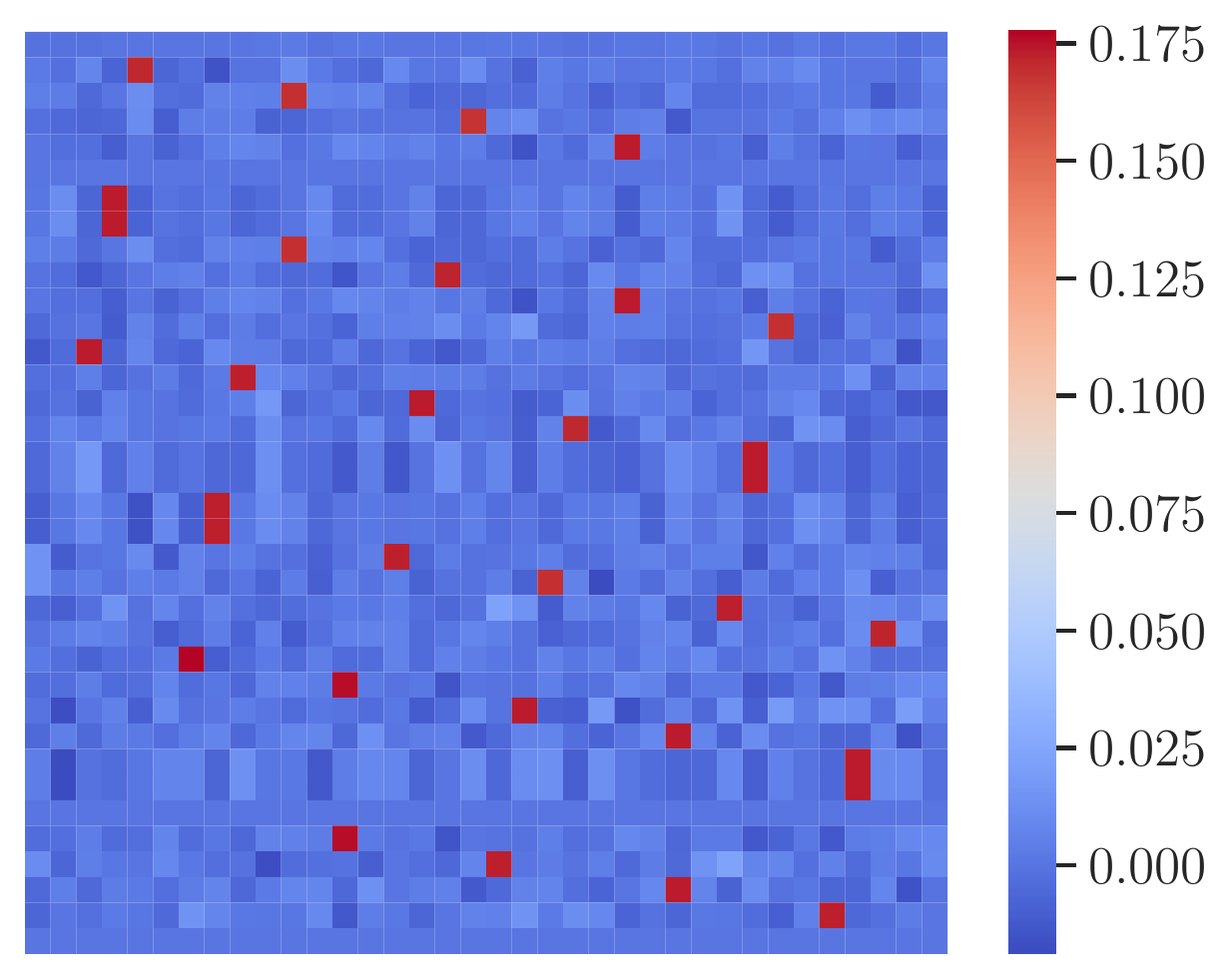}
    \caption{$\theta = 60\degree$}
  \end{subfigure}
  \begin{subfigure}[h]{0.24\textwidth}
    \centering
    \includegraphics[width=\textwidth]{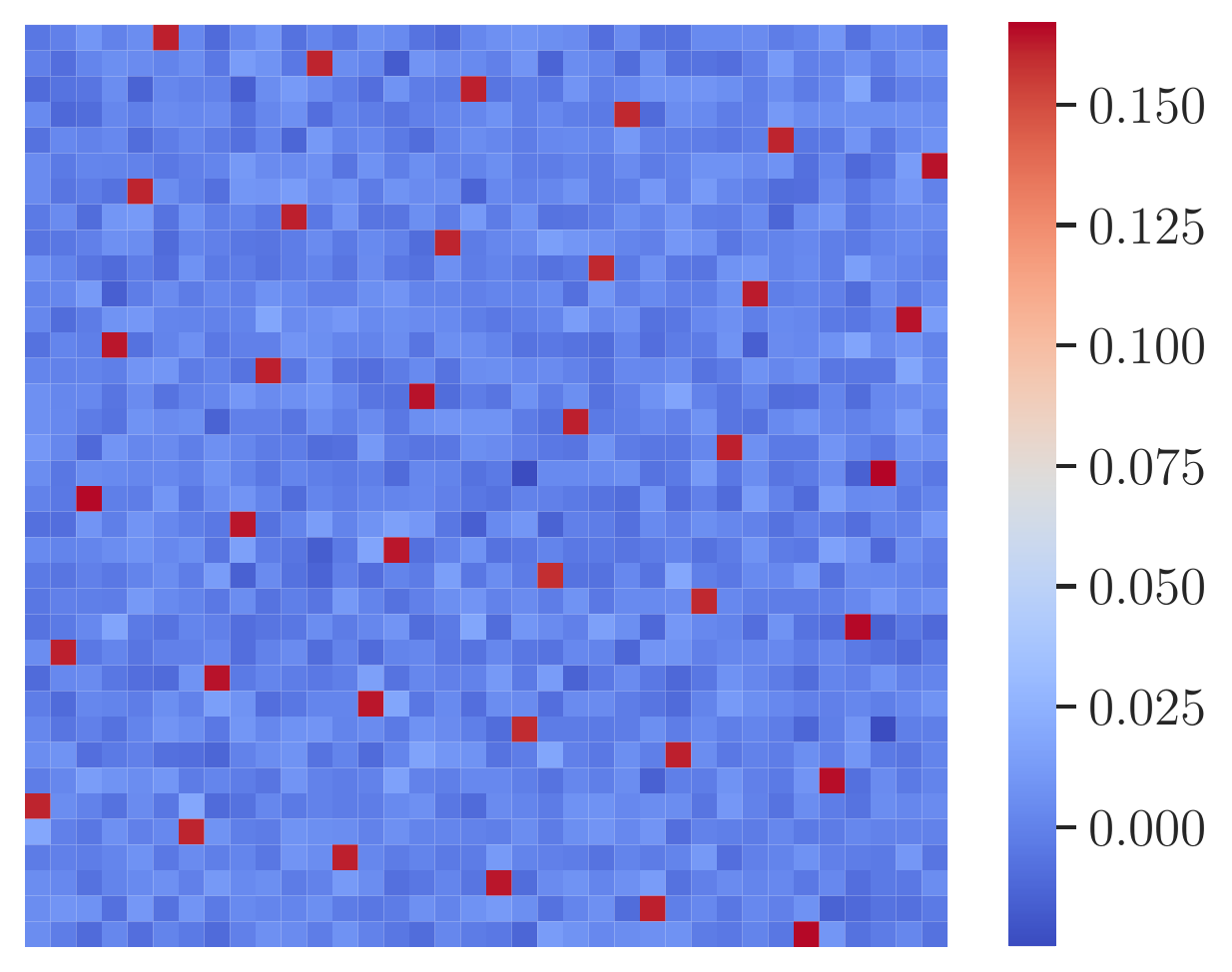}
    \caption{$\theta = 90\degree$}
  \end{subfigure}
  \caption{Learned group actions when rotating patches using an angle $\theta \in \{30\degree, 45\degree, 60\degree, 90\degree\}$.}
  \label{fig:rot}
\end{figure}

\subsubsection{Compositions of actions}
\label{subsubsec:comp}
While average pooling and rotations might not directly translate to the structures of \cref{fig:group_actions}, we see that they do have some correlation. This lead us to design another experiment, where we consider a composition of pooling and rotation\footnote{When $\theta$ does not correspond to an elementary rotation (i.e., $\theta \neq k \cdot 90\degree$ for some $k\in \mathbb{Z}$), interpolations occur. For this reason, in the composition we apply average pooling first before the rotation to minimize the effect of interpolating artifacts.}. The resulting actions can be seen in \cref{fig:comp}. We observe that by composing the two operations, the learned group action has the blocks of median filtering at the location grid defined by the rotation action. Reinterpreting the learned actions of \cref{fig:group_actions} under this new lense, we argue that the Toeplitz structure that we uncovered interpolates rotations and simultaneously applies pooling.

\begin{figure}[b]
  \centering
  \begin{subfigure}[h]{0.32\textwidth}
    \centering
    \includegraphics[width=\textwidth]{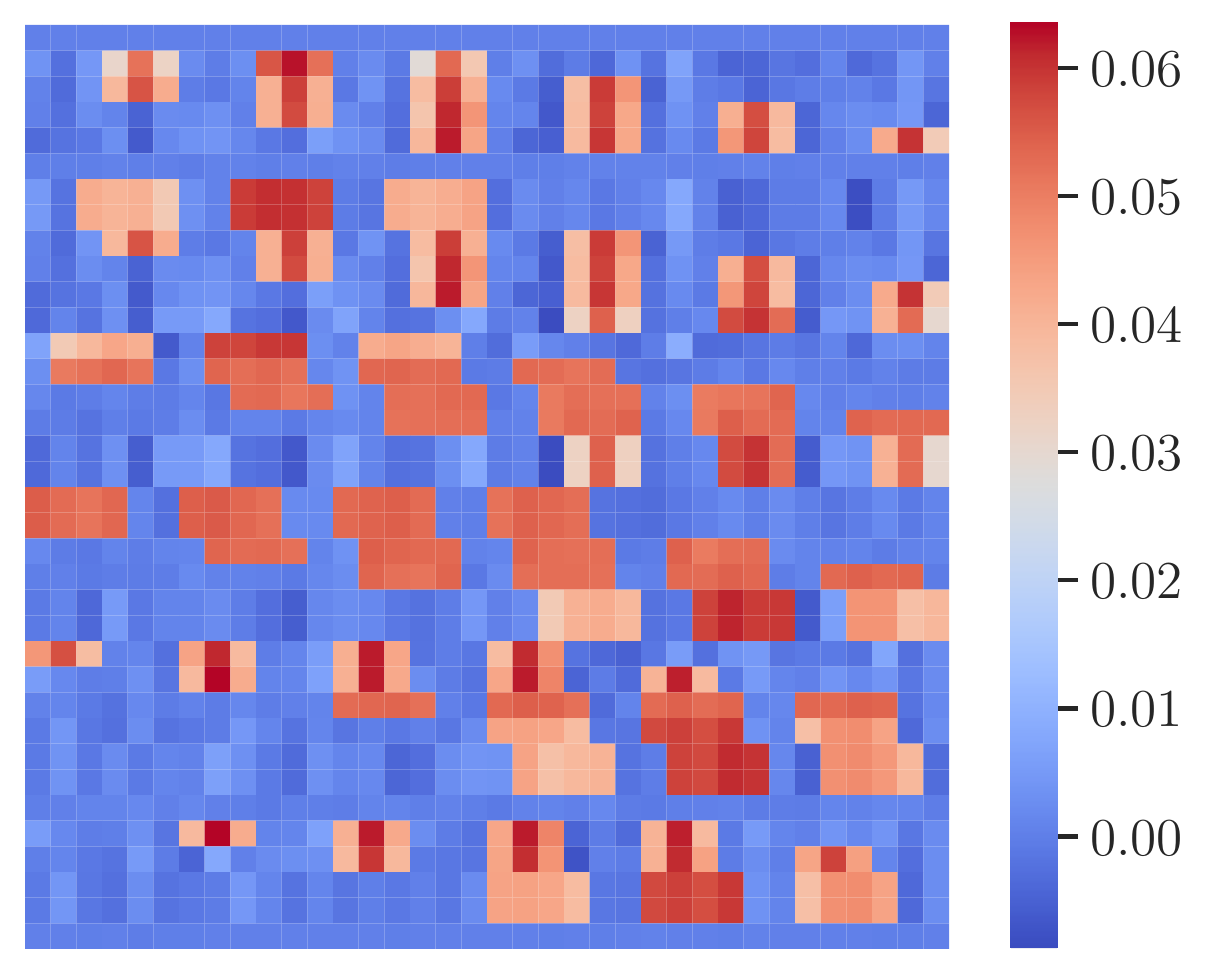}
    \caption{$r = 4$}
  \end{subfigure}
  \begin{subfigure}[h]{0.32\textwidth}
    \centering
    \includegraphics[width=\textwidth]{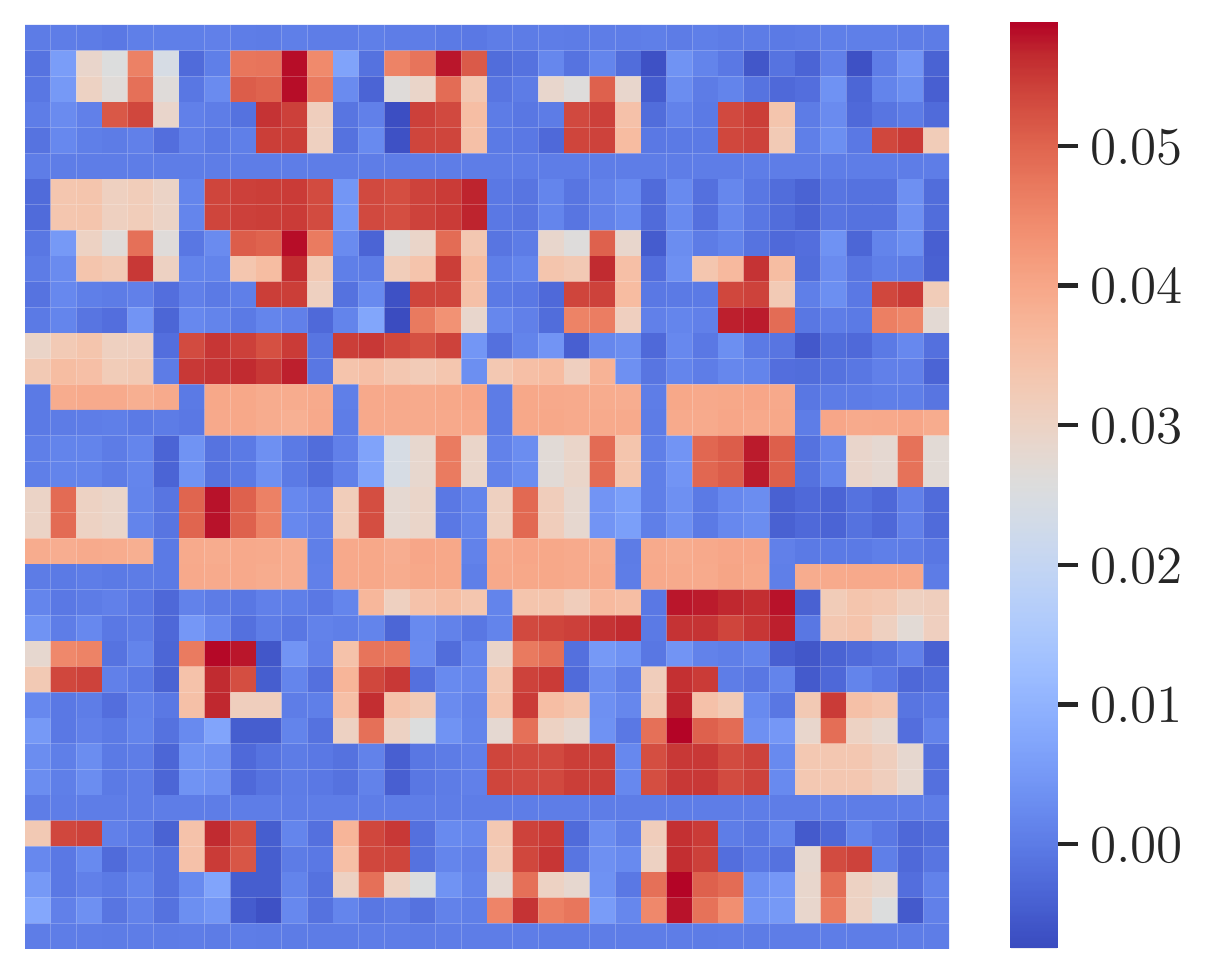}
    \caption{$r = 5$}
  \end{subfigure}
  \begin{subfigure}[h]{0.32\textwidth}
    \centering
    \includegraphics[width=\textwidth]{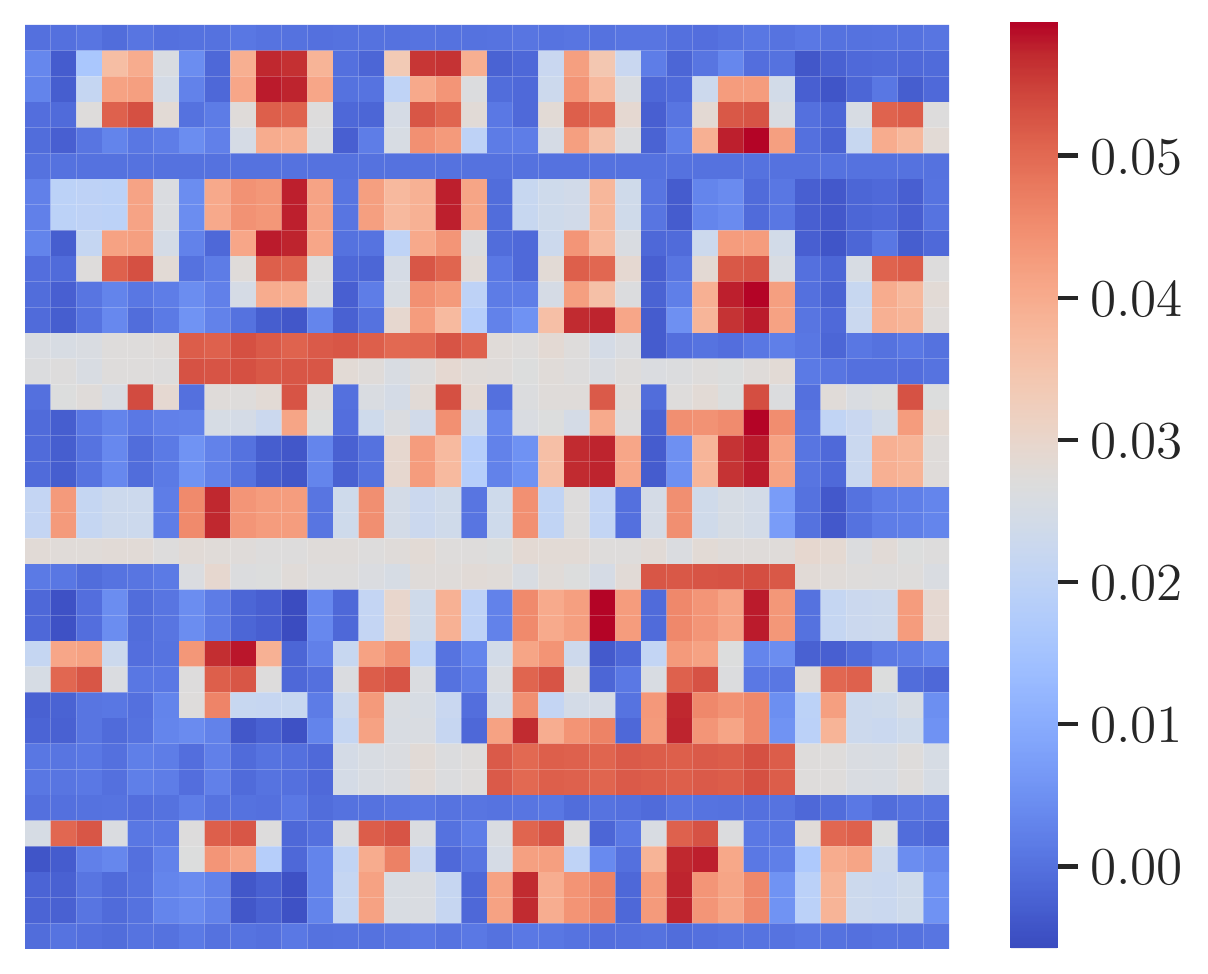}
    \caption{$r = 6$}
  \end{subfigure}
  \caption{Learned group actions when composing rotations and average pooling using a radius $r \in \{4, 5, 6\}$ (with a fixed angle $\theta = 60\degree$).}
  \label{fig:comp}
\end{figure}

\subsection{Dependence on data distribution and task}
\label{sec:data_task}
One of our theses was that the element of $\operatorname{GL}_{n\cdot m}(\mathbb{R})$ that our method learns for each group depends on both the data distribution and the downstream task. We evaluate these hypotheses below.

\subsubsection{Group actions on \texttt{MNIST}}
To evaluate the dependence of the group actions on the data distribution, we applied our method on a classification task on \texttt{MNIST}, an image dataset consisting of handwritten digits. In \cref{fig:group_mnist} we recovered the same structures as we did on \texttt{CIFAR10}, indicating that data coming from the same distribution (in this case, natural images) result in the same group actions.

\begin{figure}[t]
  \centering
  \begin{subfigure}[h]{0.32\textwidth}
    \centering
    \includegraphics[width=\textwidth]{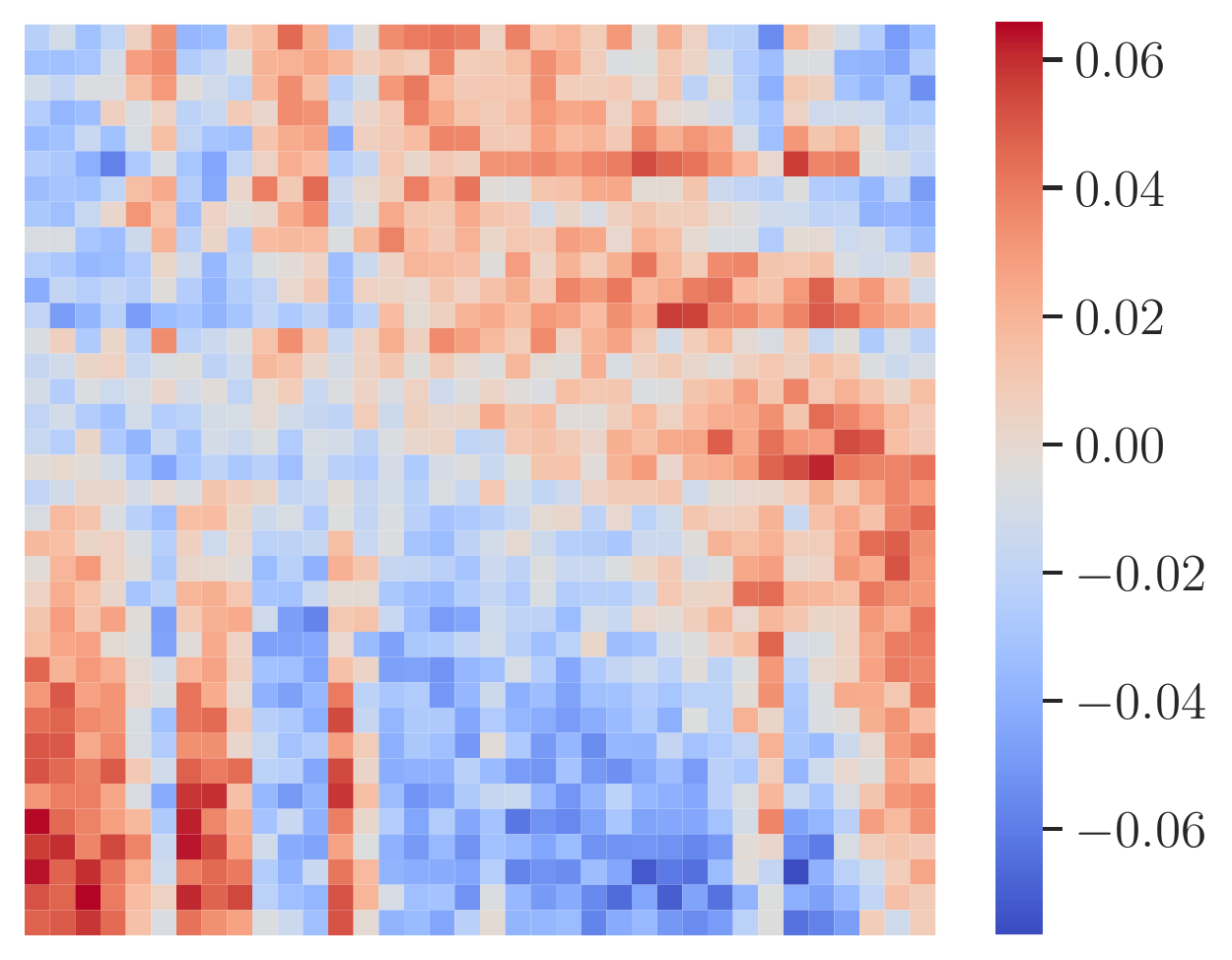}
    \caption{Skew-symmetric}
  \end{subfigure}
  \begin{subfigure}[h]{0.32\textwidth}
    \centering
    \includegraphics[width=\textwidth]{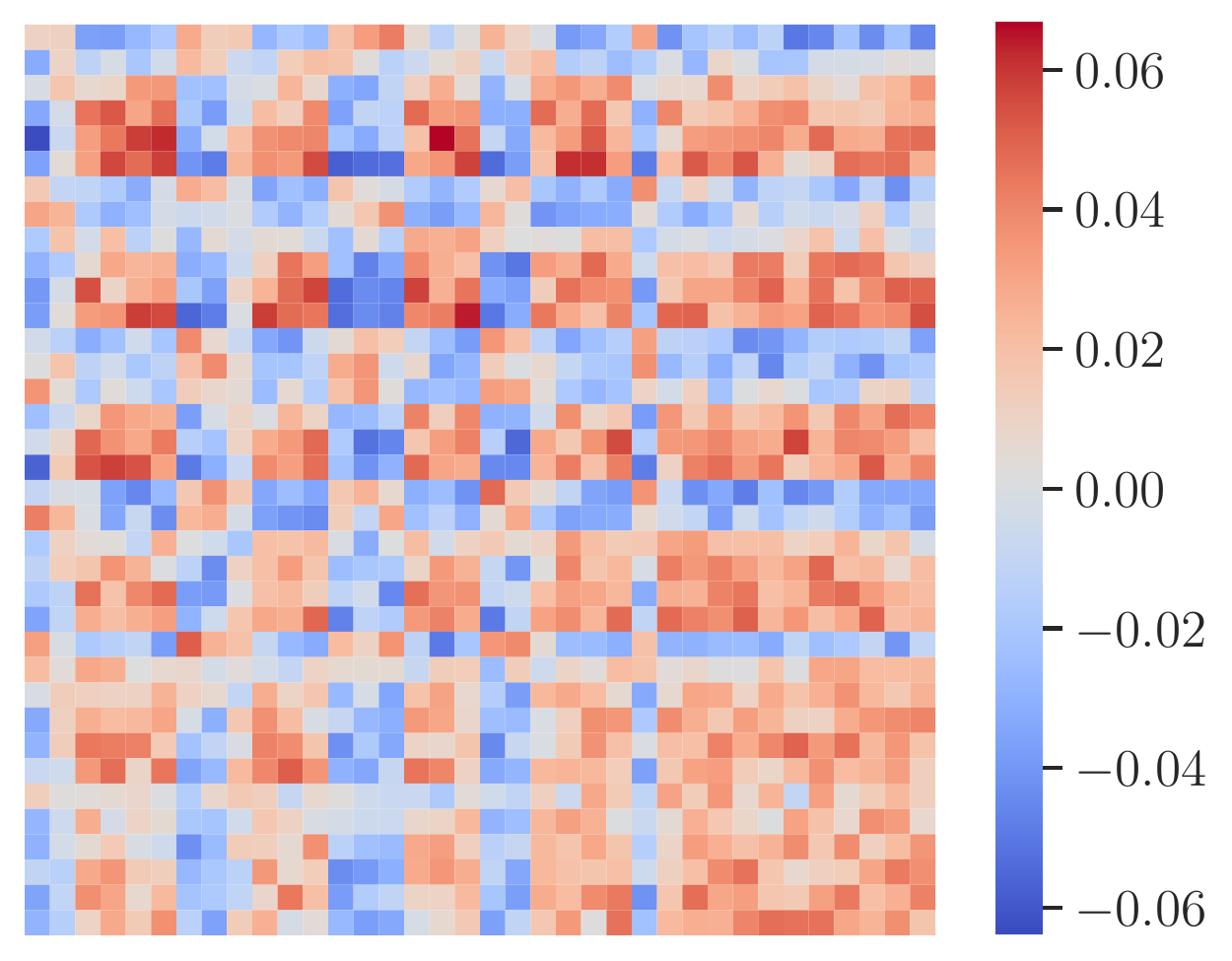}
    \caption{Toeplitz}
  \end{subfigure}
  \begin{subfigure}[h]{0.32\textwidth}
    \centering
    \includegraphics[width=\textwidth]{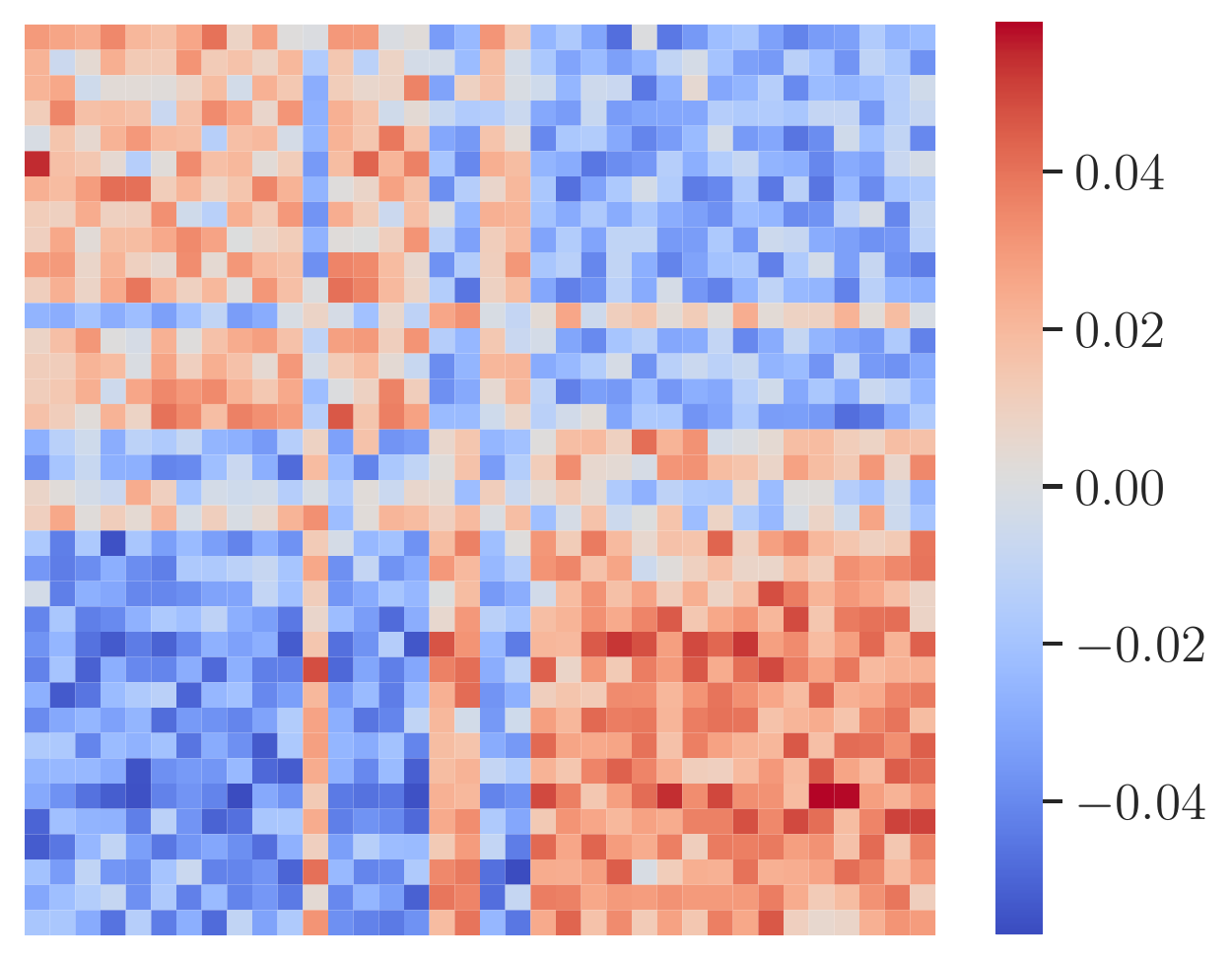}
    \caption{Multi-scale}
  \end{subfigure}
  \caption{Group actions learned on \texttt{MNIST} for classification.}
  \label{fig:group_mnist}
\end{figure}

\subsubsection{Reconstruction}
Finally, to test our hypothesis that group actions depend on the downstream task, we evaluated our method on the task of reconstructing the input. We present learned actions in \cref{fig:group_recon} for both \texttt{CIFAR10} and \texttt{MNIST}. We find that the actions for both datasets bear similarity, indicating further the dependence on the data distribution, and the emerging structures are distinctly different: contrary to actions for classification, actions for reconstruction exhibit multiple concentrated blocks.

\begin{figure}[h]
  \centering
  \begin{subfigure}[h]{0.24\textwidth}
    \centering
    \includegraphics[width=\textwidth]{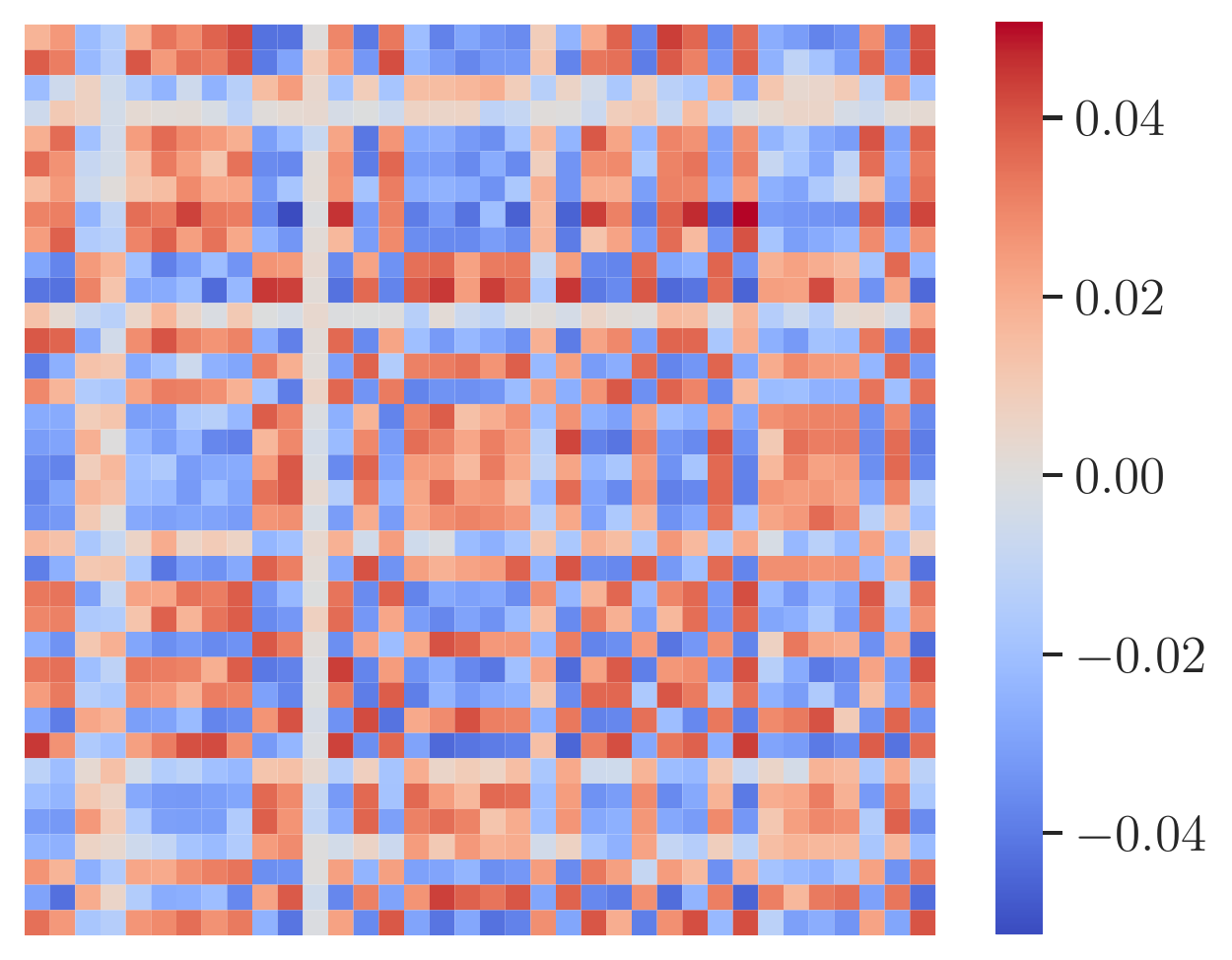}
    \caption{CIFAR}
  \end{subfigure}
  \begin{subfigure}[h]{0.24\textwidth}
    \centering
    \includegraphics[width=\textwidth]{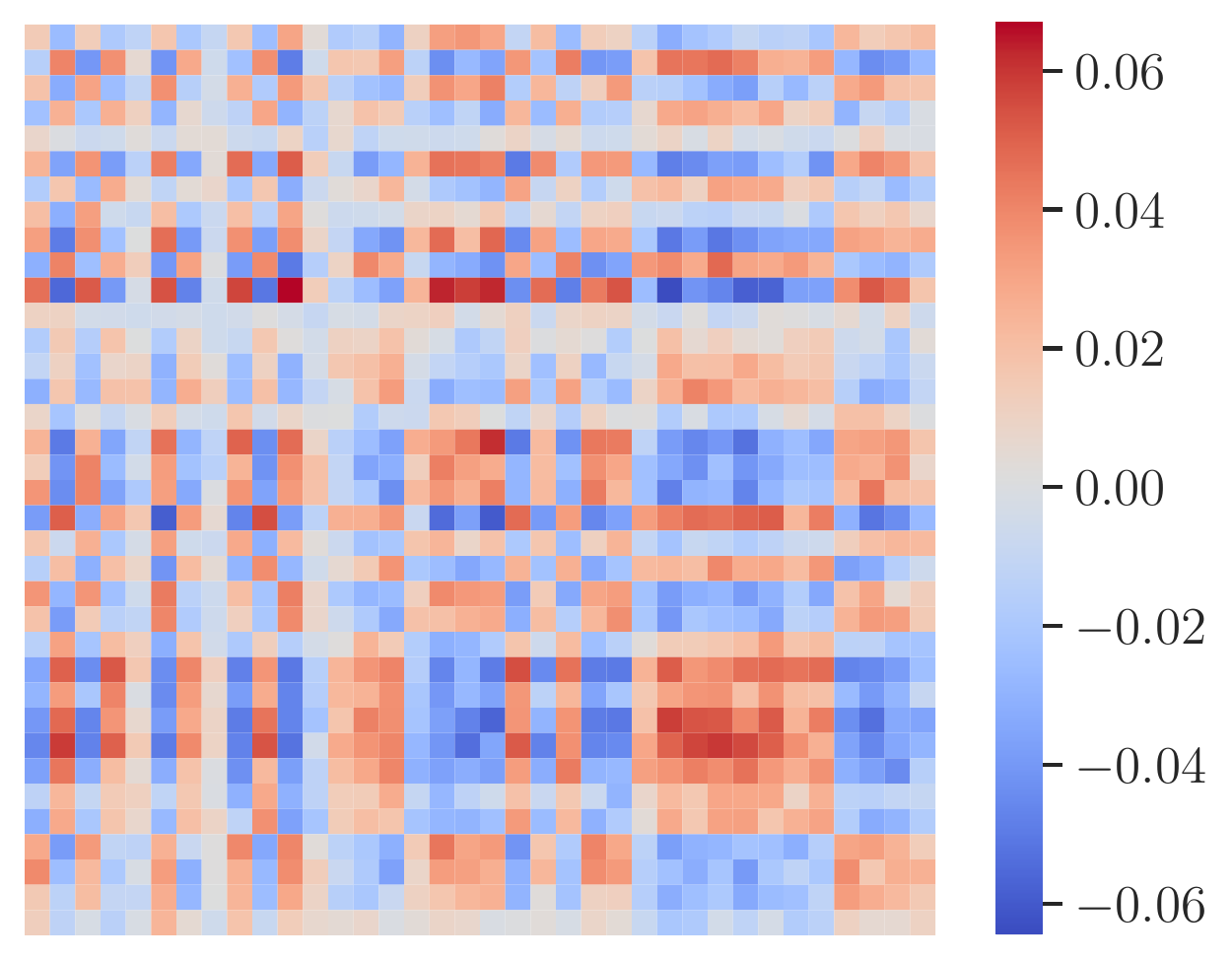}
    \caption{CIFAR}
  \end{subfigure}
  \begin{subfigure}[h]{0.24\textwidth}
    \centering
    \includegraphics[width=\textwidth]{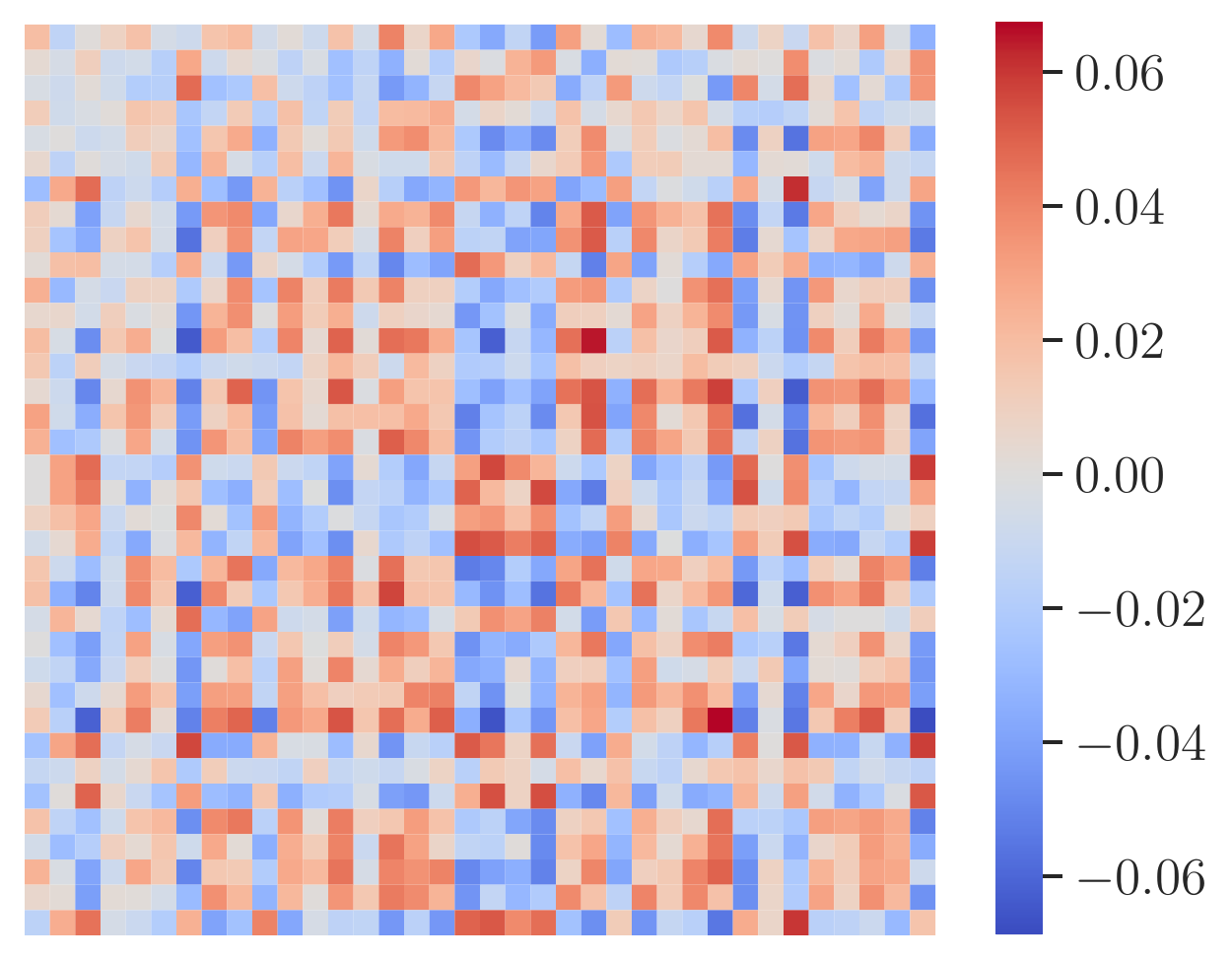}
    \caption{MNIST}
  \end{subfigure}
  \begin{subfigure}[h]{0.24\textwidth}
    \centering
    \includegraphics[width=\textwidth]{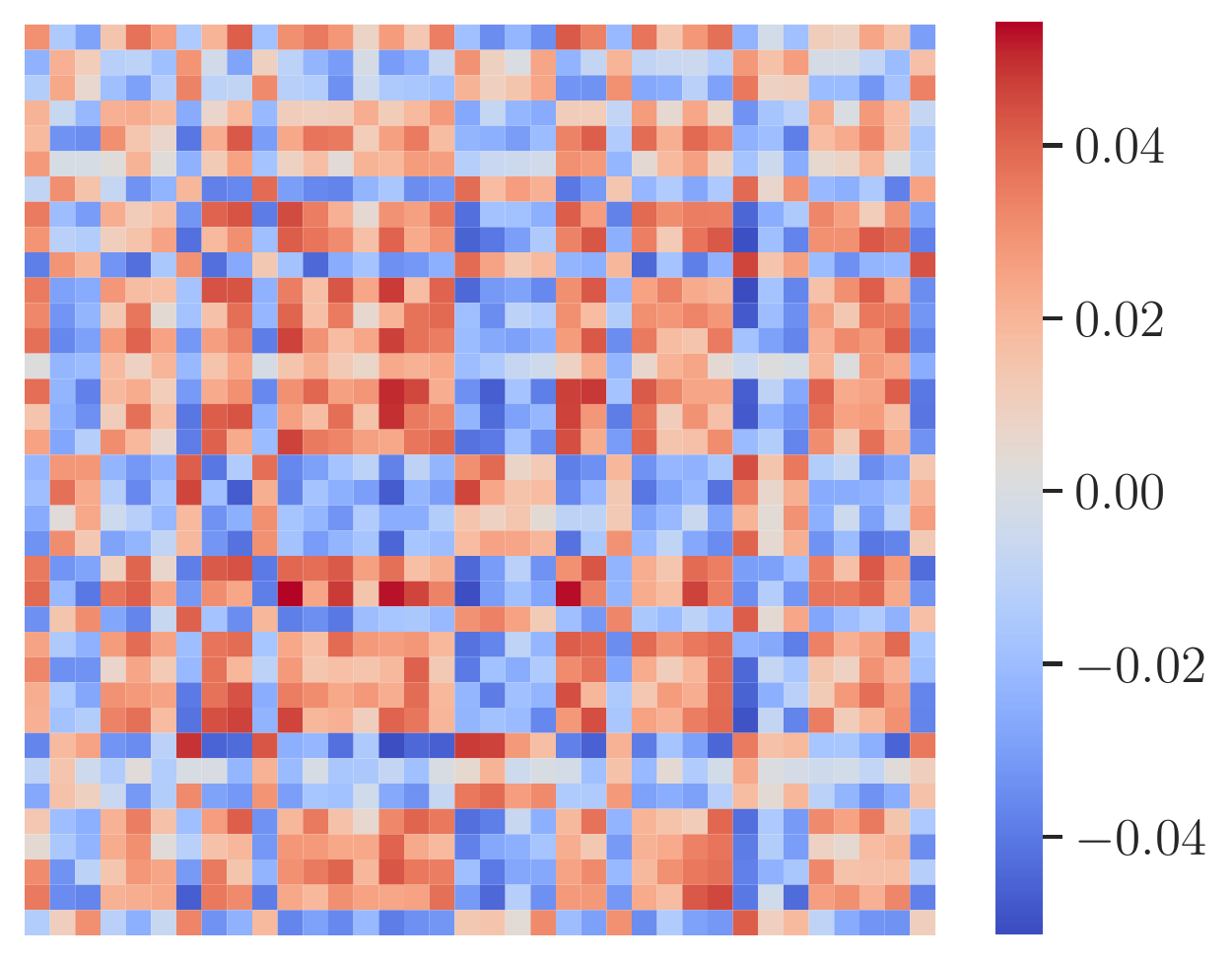}
    \caption{MNIST}
  \end{subfigure}
  \caption{Group actions learned on CIFAR and MNIST when trained for reconstruction.}
  \label{fig:group_recon}
\end{figure}

\section{Conclusion}
\label{sec:concl}
We proposed a method, Linear Group Networks (LGNs) for learning linear groups acting on the weights of neural networks. We showed that matrices in the weight space are unable to learn interesting group actions and propose a formulation where the groups act on the lifted space of the vectorized weights, providing theoretical guarantees. In our experiments, we applied our framework on datasets of natural images for different downstream tasks. We discover several interesting groups whose actions are \emph{skew-symmetric}, \emph{Toeplitz}, or \emph{multi-scale}. We drew analogues between these actions and well-known operations in machine learning, such as average pooling and planar rotations.  Our work provides a simple, minimal, and extensible framework for symmetry learning by considering linear groups in the weight space. Importantly, we addressed limitations of prior work, namely the inability of existing architectures to learn group structure from single-task data, and our framework is compatible with existing architectures and pipelines with minimal alterations.
\bibliography{refs.bib}
\newpage
\appendix

\section{Linear maps are matrices}
\label{appendix:linalg}
We present a classical result from linear algebra that states that any linear operator between vector spaces can be represented with a matrix.

\begin{proposition}
  Let $\phi: V \to W$ be a linear map between two vector spaces with $\operatorname{dim}{V} = m$ and $\operatorname{dim}(W) = n$. Then $\phi$ can be \emph{uniquely} mapped to a matrix $\bm{A}\in\mathbb{R}^{n\times m}$, up to similarity.
\end{proposition}
\begin{proof}
  Let $\{\bm{v}_1, \ldots, \bm{v}_m\}$ and $\{\bm{w}_1, \ldots, \bm{w}_n\}$ be bases of $V$ and $W$ respectively. Then the basis elements of $V$ have a representation under $\phi$ in $W$, $\phi(\bm{v}_j) = \sum_{i=1}^n a_{ij}\bm{w}_i$ for unique coefficients $a_{ij}$. From the coefficients for all the basis elements $\bm{v}_j$, construct the matric $\bm{A}\in\mathbb{R}^{n\times m}$ with $\bm{A} = [a_{ij}]$.

  Note that while $\operatorname{dim}(V) = m$ and $\operatorname{dim}(W) = n$, $V, W$ are not necessarily $\mathbb{R}^m$ and $\mathbb{R}^n$, respectively. However, they are isomorphic to those spaces via the maps $\phi_V, \phi_W$ where
  \begin{align}
    \begin{split}
      \phi_{V}: \quad& V \to \mathbb{R}^{m}\\
      & \sum_{i=1}^m \alpha_i \bm{v}_i \mapsto [\alpha_1, \ldots, \alpha_m]^T,
    \end{split}
  \end{align}
  and $\phi_W$ is defined analogously. Now consider a linear map $g: \mathbb{R}^m \to \mathbb{R}^n$ such that $\bm{x} \mapsto \bm{A}\bm{x}$, with $\bm{A}$ defined above. We will show that $g = \phi$. Indeed, note that $\phi_V(\bm{v}_j) = \bm{e}_j$, where $\bm{e}_j = [0, \ldots, 0, \underbrace{1}_{j\text{-th element}}, 0, \ldots 0]^T$ is the $j$-th basis vector of $\mathbb{R}^m$ (or $\mathbb{R}^n$ when discussing $W$), and similarly $\phi_W(\bm{w}_i) = \bm{e}_i$. Then we have
  \begin{equation*}
    g(\bm{v}_j) = \bm{A}\bm{v}_j = \bm{A}\bm{e}_j = \begin{bmatrix}
      a_{1j}\\
      \vdots\\
      a_{nj}
    \end{bmatrix} = \sum_{i=1}^n a_{ij} \bm{e}_i = \sum_{i=1}^n a_{ij} \bm{w}_i = \phi(\bm{v}_j).
  \end{equation*}
  Therefore, it follows $g = \phi$. Moreover, $\bm{A}$ is unique, up to similarity. Indeed, consider change-of-basis matrices $\bm{P}$ and $\bm{Q}$ for $V$ and $W$ respectively. Then $\phi$ would map from the transformed basis of $V$ to the transformed basis of $W$
  \begin{equation*}
    \bm{Q}\phi(\bm{P}\bm{v}_j) = \bm{Q}\bm{A}\bm{P}\bm{v}_j.
  \end{equation*}
  However, the matrix $\bm{B} = \bm{Q}\bm{A}\bm{P}$ is similar to $\bm{A}$ as $\bm{Q}$ and $\bm{P}$ are invertible by definition, and so the transformation is uniquely defined by $\bm{A}$.
\end{proof}

\section{Discussion of the invertibility loss}
\label{appendix:loss}
In \cref{subsec:invert} we suggested the introduction of auxiliary matrices $\widetilde{\bm{A}}_{(k, l)}$, to be used \emph{only} during training, to promote the invertibility of the group actions parametrized by $\bm{A}_{(k, l)}$. 
\begin{wrapfigure}{r}{0.45\textwidth}
  \centering
  \includegraphics[width=0.45\textwidth]{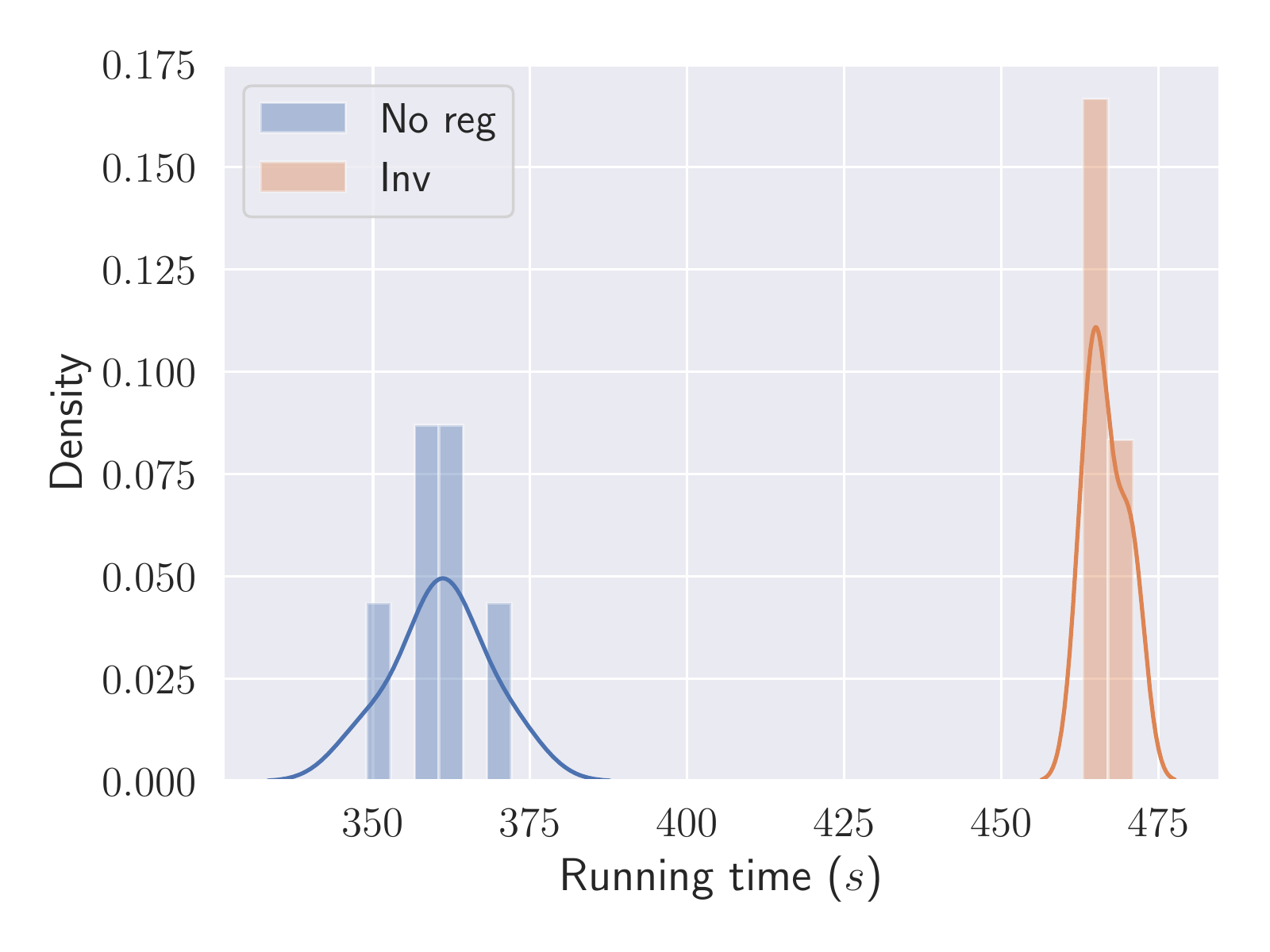}
  \caption{Effect of the regularization on runtime.}
  \label{fig:runtime}
\end{wrapfigure}
One could avoid the introduction of $\widetilde{\bm{A}}_{(k, l)}$ as a candidate for the inverse of $\bm{A}_{(k, l)}$ by considering that $\bm{A}_{(k, l)}$ is orthogonal (and thus its inverse is $\bm{A}_{(k, l)}^T$). However, that further restricts the weight matrices to belong to the Stiefel manifold $\operatorname{St}(m\cdot n, m\cdot n)$; while $\operatorname{St}(m\cdot n, m\cdot n)$ is a subgroup of $\operatorname{GL}_{m\cdot n}(\mathbb{R})$, it's a significantly more restricted set and in this work we wanted to avoid that. The introduction of $\widetilde{\bm{A}}_{(k, l)}$ results in a moderate increase in training time of about $30\%$, as can be seen in \cref{fig:runtime}, while requiring extra storage during training to store the matrices $\widetilde{\bm{A}}_{(k, l)}$.

Another way to promote invertibility without the need for the extra training time matrices $\widetilde{\bm{A}}_{(k, l)}$ is by penalizing the singular values of $\bm{A}_{(k, l)}$ to be away from zero. To this end, we could consider the loss
\begin{equation}
  \label{eq:svd}
  L = -\mu \sum_{i, k, l}\sigma_i(\bm{A}_{(k, l)}),
\end{equation}
where $\sigma_i(\bm{A}_{(k, l)})$ is the $i$-th singular value of $\bm{A}_{(k, l)}$ and $\mu$ is a regularization parameter. Alternative forms of \eqref{eq:svd} can be considered. For $\bm{A}_{(k, l)}$ to be invertible \emph{all} singular values need to be bounded away from zero; to that end, a loss of the form $-\log \prod_{i,k,l}\sigma_i(\bm{A}_{(k, l)})$ might be more apropriate, as it heavily penalizes all singular values to be bound away from zero. In our experiments, we found that the formulation of \eqref{eq:svd} was adequate to bound the singular values away from zero. Conventional wisdom is that SVD is costly and unstable when used for deep learning; when used in practice for our application we observed numerical stability (via the use of \texttt{svdvals}) and moderate increases of about $2\times$ in training time (by moving the SVD-related computations to the CPU\footnote{As suggested by \href{https://github.com/KingJamesSong/DifferentiableSVD/tree/main}{https://github.com/KingJamesSong/DifferentiableSVD/tree/main}.}). Due to the increased training time, we opted for the formulation of \cref{subsec:invert} for the bulk of our experiments.

\section{Architecture, datasets, and hyperparameters}
\label{appendix:architecture}
\subsection*{MNIST}
The \texttt{MNIST}\footnote{\href{http://yann.lecun.com/exdb/mnist/}{http://yann.lecun.com/exdb/mnist/}.} dataset consists of $28\times 28$ black and white images of handwritten digits with slight variations in orientation and writing style. The images are size-normalized and have been centered. There are $60,000$ images for training and $10,000$ images for testing.

\subsection*{CIFAR10}
The \texttt{CIFAR10}\footnote{\href{https://www.cs.toronto.edu/~kriz/cifar.html}{https://www.cs.toronto.edu/~kriz/cifar.html}.} dataset consists of $32\times 32$ color images in $10$ classes: airplane, automobile, bird, cat, deer, dog, frog, horse, ship, and truck. There are $50,000$ images for training and $10,000$ images for testing.

\subsection*{Architecture}
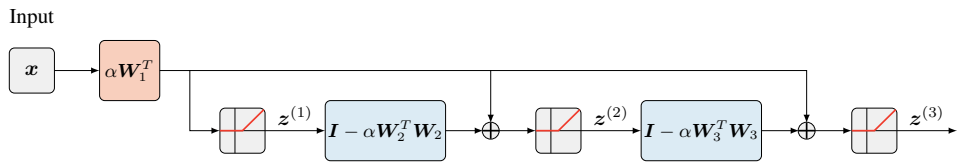
\begin{figure}[h!]
  \centering
   \scalebox{0.8}{\begin{tikzpicture}
          \draw[fill=light_gray, rounded corners=2.5pt] (-1.5, 0.125) rectangle (-0.75, 0.875);
          \node at (-1.125, 0.5) {$\bm{x}$};
          \node at (-1.125, 1.375) {Input};
          \draw[-latex] (-0.75, 0.5) -- (0, 0.5);

          \draw[fill=neur_green, rounded corners=2.5pt] (0, 0) rectangle (1, 1);
          \node at (0.5, 0.5) {{\footnotesize $\alpha\bm{W}_1^T$}};
          \draw[-latex] (1, 0.5) -- (1.5, 0.5) -- (1.5, -0.5) -- (2, -0.5);
    \draw[-latex] (1.5, 0.5) -- (6.5, 0.5) -- (6.5, -0.375);
    \draw[-latex] (6.5, 0.5) -- (11.75, 0.5) -- (11.75, -0.375);

    \draw[fill=light_gray, rounded corners=2.5pt] (2, -0.875) rectangle (2.75, -0.125);
    \draw[line width=0.01mm] (2, -0.5) -- (2.75, -0.5);
    \draw[line width=0.01mm] (2.25, -0.875) -- (2.25, -0.125);
    \draw[outt_purpl, thick] (2, -0.5) -- (2.375, -0.5) -- (2.73, -0.15);
          \node at (3.25, -0.25) {$\bm{z}^{(1)}$};
          \draw[-latex] (2.75, -0.5) -- (3.75, -0.5);

          \draw[fill=neur_purpl, rounded corners=2.5pt] (3.75, -1) rectangle (5.75, 0);
          \node at (4.75, -0.5) {{\footnotesize $\bm{I} - \alpha\bm{W}_2^T\bm{W}_2$}};
    \draw[-latex] (5.75, -0.5) -- (6.375, -0.5);
          \draw[fill=light_gray] (6.5, -0.5) circle (4pt);
    \draw[thick] (6.5, -0.365) -- (6.5, -0.635);
    \draw[thick] (6.36, -0.5) -- (6.64, -0.5);

    \draw[-latex] (6.625, -0.5) -- (7.25, -0.5);
    \draw[fill=light_gray, rounded corners=2.5pt] (7.25, -0.875) rectangle (8, -0.125);
    \draw[line width=0.01mm] (7.25, -0.5) -- (8, -0.5);
    \draw[line width=0.01mm] (7.5, -0.875) -- (7.5, -0.125);
    \draw[outt_purpl, thick] (7.25, -0.5) -- (7.625, -0.5) -- (7.98, -0.15);
    \node at (8.5, -0.25) {$\bm{z}^{(2)}$};
    \draw[-latex] (8, -0.5) -- (9, -0.5);

    \draw[fill=neur_purpl, rounded corners=2.5pt] (9, -1) rectangle (11, 0);
          \node at (10, -0.5) {{\footnotesize $\bm{I} - \alpha\bm{W}_3^T\bm{W}_3$}};
    \draw[-latex] (11, -0.5) -- (11.625, -0.5);
          \draw[fill=light_gray] (11.75, -0.5) circle (4pt);
    \draw[thick] (11.75, -0.365) -- (11.75, -0.635);
    \draw[thick] (11.61, -0.5) -- (11.89, -0.5);

    \draw[-latex] (11.875, -0.5) -- (12.5, -0.5);
    \draw[fill=light_gray, rounded corners=2.5pt] (12.5, -0.875) rectangle (13.25, -0.125);
    \draw[line width=0.01mm] (12.5, -0.5) -- (13.25, -0.5);
    \draw[line width=0.01mm] (12.75, -0.875) -- (12.75, -0.125);
    \draw[outt_purpl, thick] (12.5, -0.5) -- (12.875, -0.5) -- (13.23, -0.15);
    \node at (13.75, -0.25) {$\bm{z}^{(3)}$};
    \draw[-latex] (13.25, -0.5) -- (14.25, -0.5);
      \end{tikzpicture}}
  \caption{The unfolded architecture of \eqref{eq:ista} for $L = 3$.}
  \label{fig:neural}
\end{figure}
\cref{fig:neural} demonstrates the architecture implied by \eqref{eq:ista}. The matrices $\bm{W}_i$ are those defined in \eqref{eq:layer_weights}. In our case where models are mainly convolutional, the matrix products are replaced with correlations and convolutions.

\subsection*{Training and hyperparameters}
We use the PyTorch framework for all our experiments were run on a NVIDIA GeForce RTX 3090 Ti. We use the Adam optimizer with a learning rate of $0.01$, and half that learning rate when training is $50\%$, $75\%$, and $87.5\%$ complete. We train all our models for $100$ epochs.

As stated in the main text of our paper, we consider filters of size $6\times 6$, have $L=4$ layers in our architectures, and each layer has $K=5$ groups of $p=4$ elements each. When using the regularization of \cref{subsec:invert}, a value of $\mu = 0.001$ was used when $\widetilde{\bm{A}}_{(k, l)}$ was used and a value of $\mu = 0.01$ when the regularization used \texttt{svdvals}.

We used a trainable bias $\lambda$ in \eqref{eq:soft_thresh} for each filter and batch normalization was used after every layer (except for the layer before the classifier), following best practices. The step size $\alpha$ for ISTA in \eqref{eq:ista} was chosen to be $0.01$ and \emph{average pooling} was used right before the classification network to reduce the computational complexity.

\section{Diagonalization of circulant matrices}
\label{appendix:fourier}
Circulant matrices are a special case of Toeplitz matrices, where the upper and lower minor diagonals are enforced to be related and follow a ``circular'' structure. In contrast, the upper and lower minor diagonals in Toeplitz matrices are generally independent. Circulant matrices are of importance as they are diagonalized by the DFT matrix, which significantly speeds up computations. in other words if a matrix $\bm{C}$ is circulant, then
\begin{equation*}
  \bm{F}_n \bm{C} \bm{F}_n^{-1} = \bm{D},
\end{equation*}
where $\bm{D}$ indicates a diagonal matrix and $\bm{F}_n$ is the $n\times n$ DFT matrix. Toeplitz matrices are only asymptotically circulant if one considers infinite repetitions of the Toeplitz matrix.
\begin{wrapfigure}{r}{0.6\textwidth}
  \centering
  \begin{subfigure}[h]{0.25\textwidth}
    \centering
    \includegraphics[width=\textwidth]{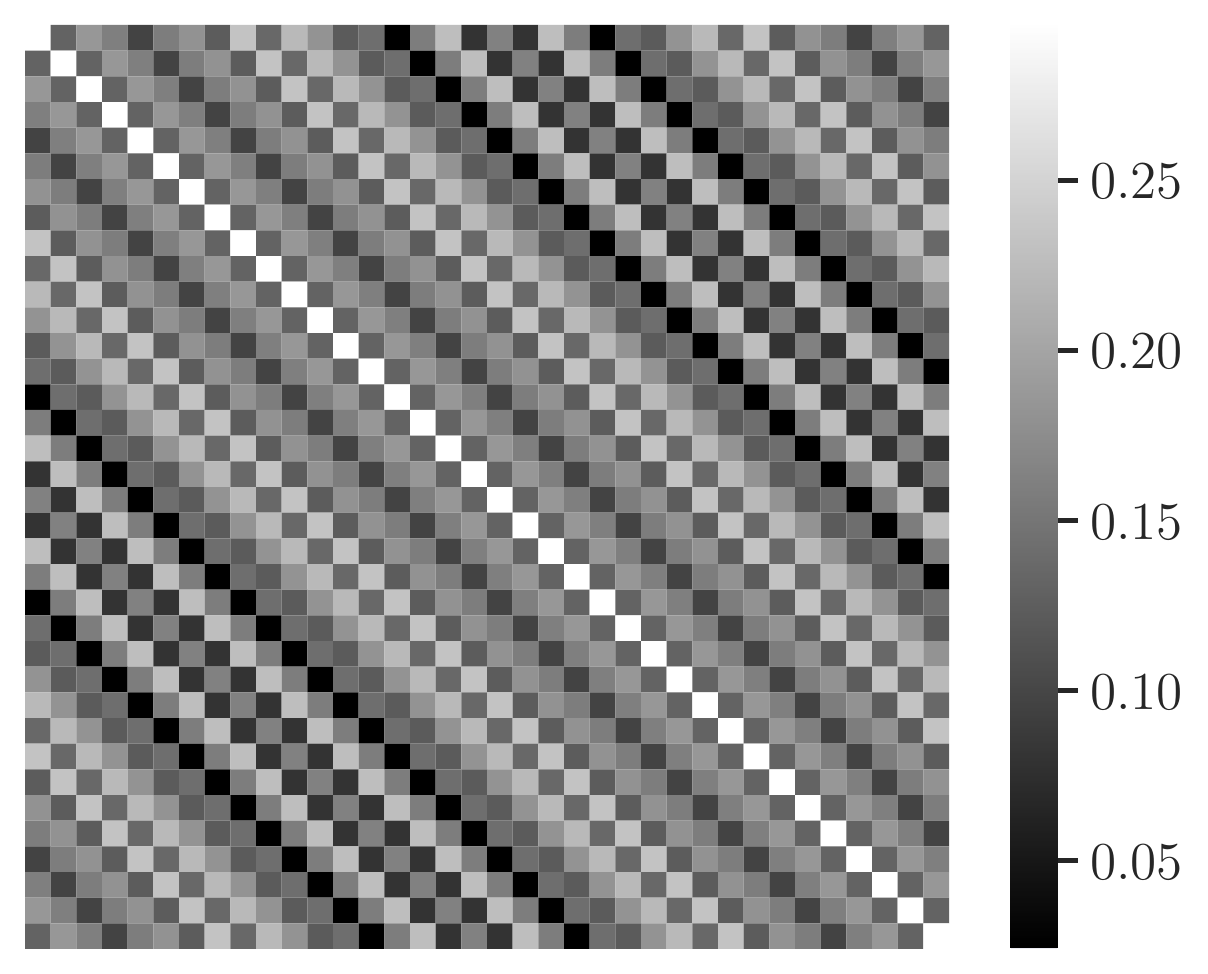}
    \caption{Circulant}
  \end{subfigure}
  \begin{subfigure}[h]{0.25\textwidth}
    \centering
    \includegraphics[width=\textwidth]{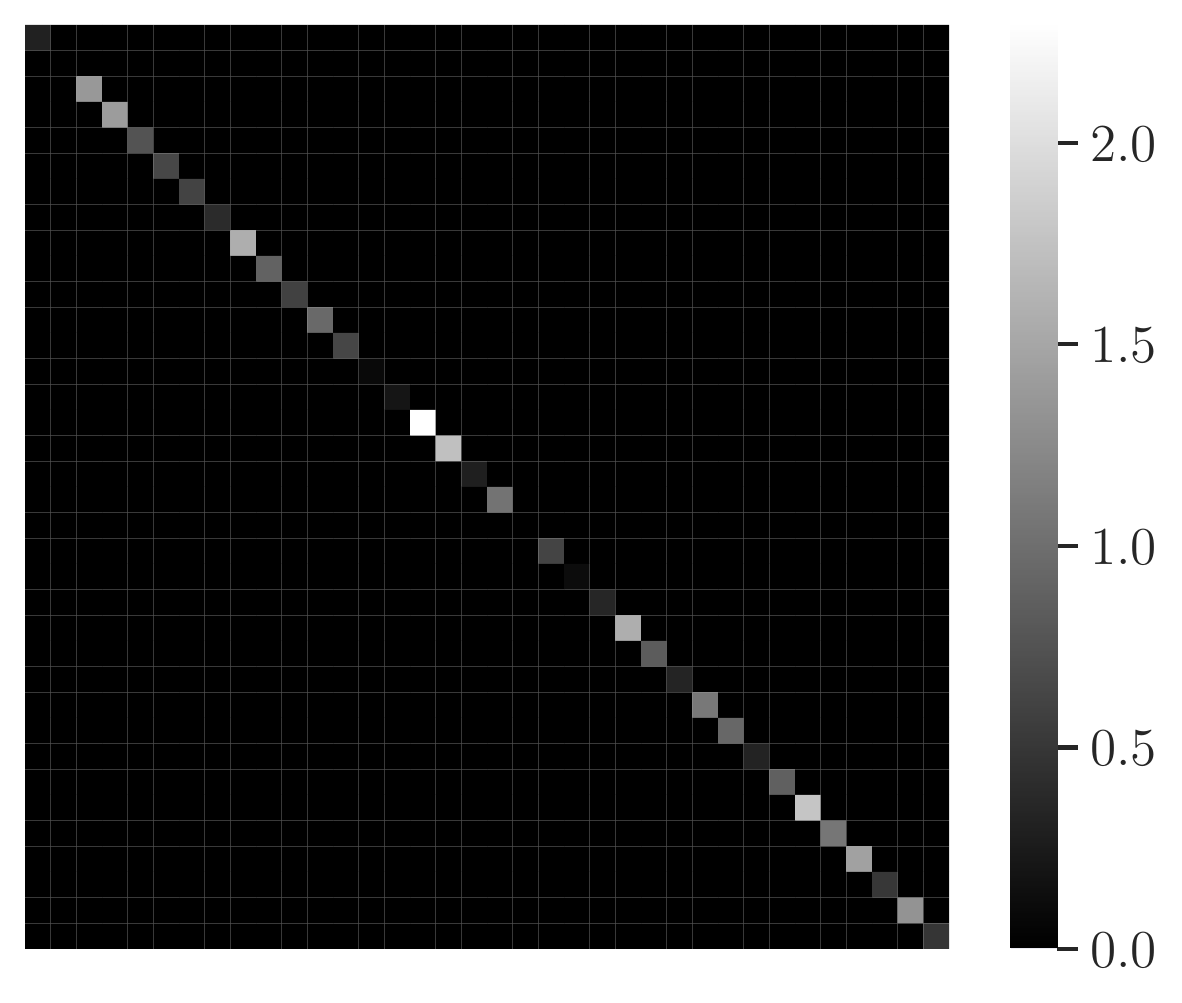}
    \caption{Diagonal}
    \label{fig:diagonal}
  \end{subfigure}
  \caption{Circulant matrices are diagonalized by $\bm{F}_n$.}
  \label{fig:circ_diag}
\end{wrapfigure}
In this appendix we examine the extent to which the matrices of \cref{fig:group_actions} are actually Toeplitz, i.e., they are diagonalized by the matrix of the Discrete Fourier Transform. As a baseline, in \cref{fig:circ_diag}, we consider a random diagonal matrix and generate the corresponding circulat matrix as $\bm{F}_n^{-1} \bm{D} \bm{F}_n$. We clearly observe that the structure is circulant.

We next turn our attention to the group action of \cref{fig:toeplitz}. We attempt to diagonalize this action in \cref{fig:diag_act}, where we also include the attempt to diagonalize a random matrix with Gaussian entries.
\begin{figure}[h]
  \begin{subfigure}[h]{0.24\textwidth}
    \centering
    \includegraphics[width=\textwidth]{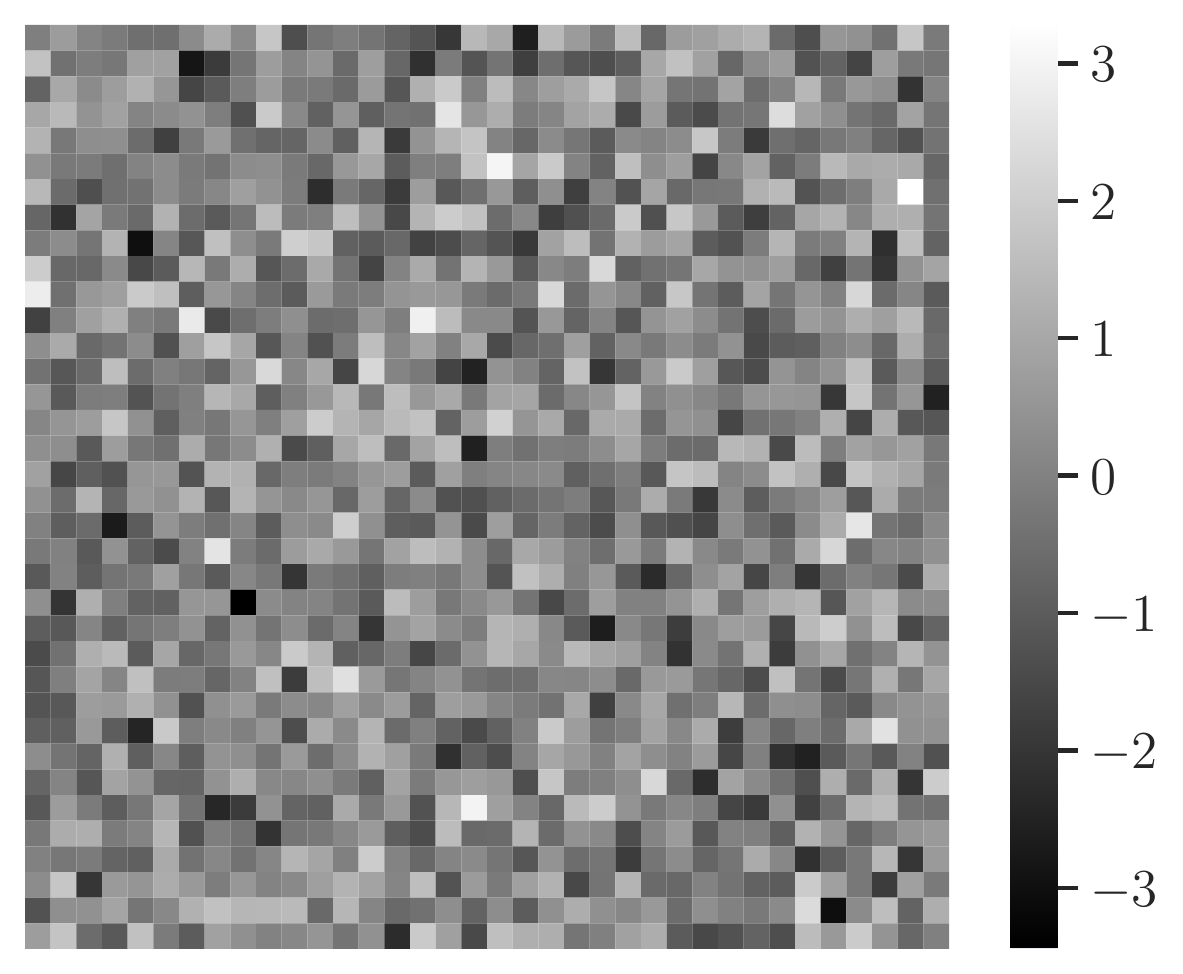}
    \caption{Random}
  \end{subfigure}
  \begin{subfigure}[h]{0.24\textwidth}
    \centering
    \includegraphics[width=\textwidth]{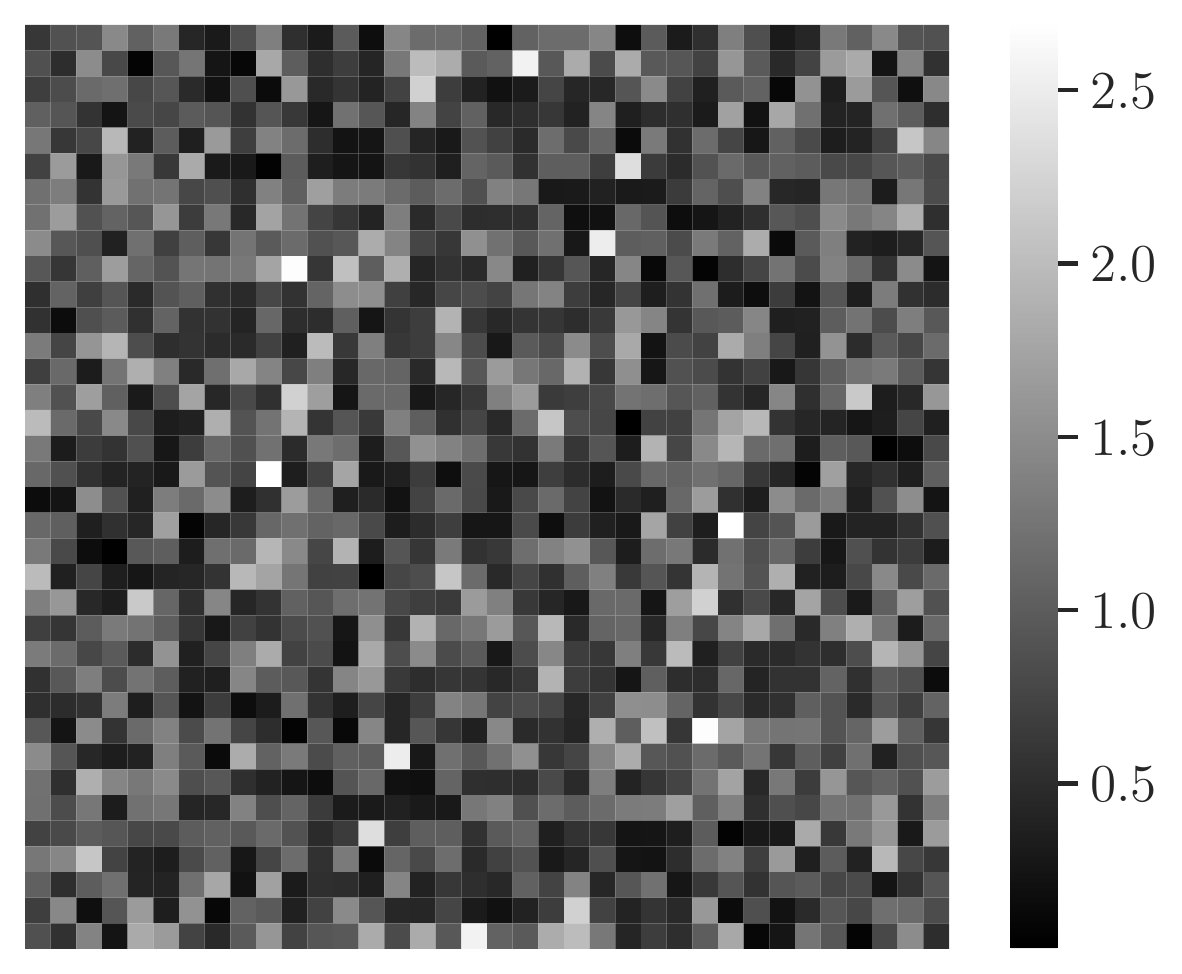}
    \caption{Diagonalization}
    \label{fig:rand_diag}
  \end{subfigure}
  \begin{subfigure}[h]{0.24\textwidth}
    \centering
    \includegraphics[width=\textwidth]{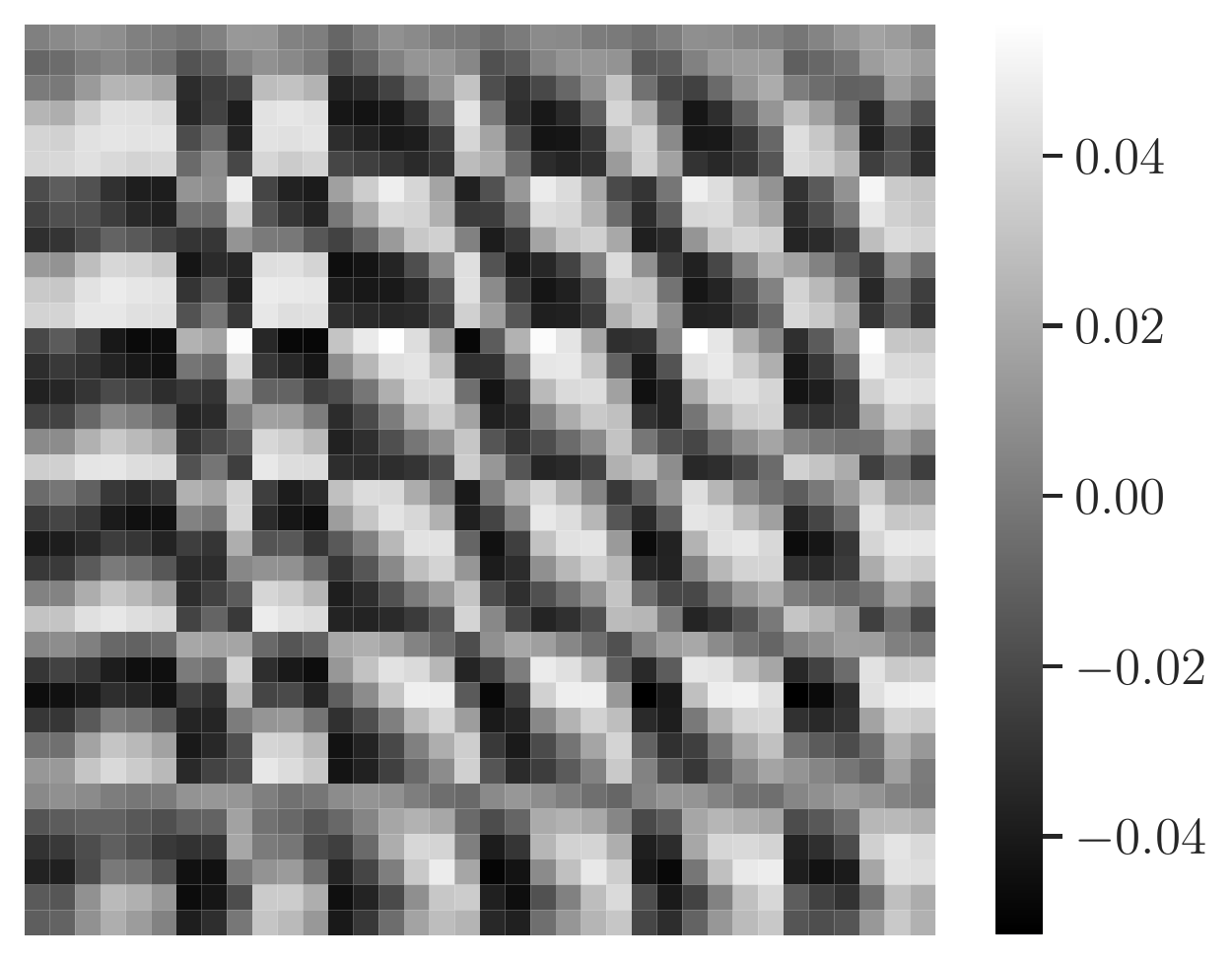}
    \caption{Group action}
  \end{subfigure}
  \begin{subfigure}[h]{0.24\textwidth}
    \centering
    \includegraphics[width=\textwidth]{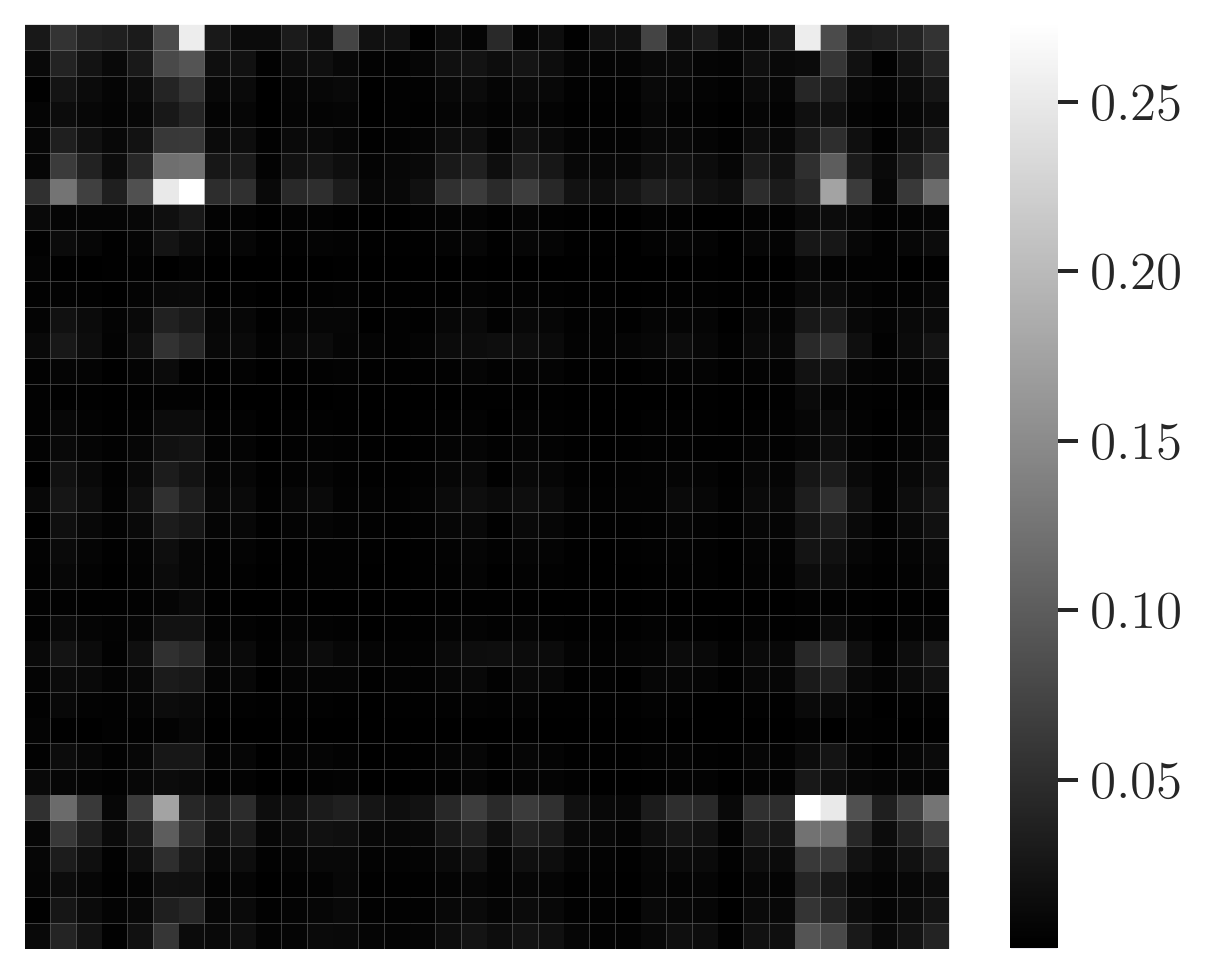}
    \caption{Diagonalization}
    \label{fig:diagonalizable_action}
  \end{subfigure}
  \caption{Attempt to diagonalize a random matrix and the group action of \cref{fig:toeplitz}.}
  \label{fig:diag_act}
\end{figure}
We observe that, while the structure $\bm{F}_n \bm{A}_{(k, l)} \bm{F}_n^{-1}$ is not strictly diagonal, it has very few non zero values and they are concentrated in a few rows and columns. As hinted at the beginning of the section, Toeplitz matrices are only asymptotically diagonalizable by the DFT matrix, so this finding is not contradictory with our claims. We contrast the diagonal structure of $\bm{A}_{(k, l)}$ with that of a random matrix: we observe that an attempt to diagonalize random matrices fails and the result is highly random itself. The degree to which a real, learnable matrix is diagonalizable by $\bm{F}_n$ is a spectrum: on the one end of the spectrum we have \cref{fig:rand_diag}, which is not diagonalized at all by $\bm{F}_n$, and on the other end \cref{fig:diagonal}, which is completely diagonalized. \cref{fig:diagonalizable_action} lies in between these extremes, and we argue it is closer to being circulant than it is to being random.

\end{document}